\documentclass[10pt,twocolumn,letterpaper]{article}

\usepackage[pagenumbers]{cvpr} 

\definecolor{cvprblue}{RGB}{219, 61, 111}
\usepackage[pagebackref,breaklinks,colorlinks,allcolors=cvprblue]{hyperref}

\def\paperID{43079} 
\def\confName{CVPR}
\def\confYear{2026}

\title{Globally Optimal Pose from Orthographic Silhouettes}

\author{\normalsize{
Agniva Sengupta$^{1,2}$ \quad
Dilara Kuş$^{1,2}$ \quad
Jianning Li$^{2}$ \quad
Stefan Zachow$^{2}$} \\
\vspace{0.5em}
\footnotesize{$^{1}$Freie Universität Berlin \qquad
$^{2}$Zuse Institute Berlin }
}
\date{}

\usepackage{amsmath,amsfonts}
\usepackage{algorithm}
\usepackage{array}
\usepackage{textcomp}
\usepackage{stfloats}
\usepackage{url}
\usepackage{verbatim}
\usepackage{algpseudocode}
\usepackage{graphicx}
\usepackage{cite}
\usepackage[acronym]{glossaries}
\usepackage{amsthm}
\usepackage[abs]{overpic}
\usepackage[T1]{fontenc}
\usepackage{changepage}
\usepackage{ragged2e, lipsum}
\usepackage{tabularx, caption, booktabs}
\usepackage{fancyhdr}
\usepackage[pagebackref,breaklinks]{hyperref}
\usepackage[capitalize]{cleveref}
\usepackage[utf8]{inputenc}
\usepackage{tabularray}
\usepackage{colortbl}
\usepackage{mathdots}
\usepackage{makecell}
\usepackage{amssymb}
\usepackage{mathabx}
\usepackage{enumitem}
\usepackage{multirow}
\usepackage{subcaption}
\usepackage{blindtext}
\usepackage[export]{adjustbox}
\usepackage{pifont}
\usepackage{accents}
\usepackage{hhline}
\usepackage{times}
\usepackage{epsfig}
\usepackage{graphicx}
\usepackage{amsmath}
\usepackage{amssymb}
\usepackage{mathrsfs}  
\usepackage{wrapfig}
\usepackage{xr-hyper}
\usepackage{cancel}
\usepackage{textcomp}
\usepackage[accsupp]{axessibility}  
\usepackage{titling}

\newcommand{\y}[1]{\gls{#1}}

\mathchardef\mhyphen="2D

 \definecolor{shadecolor}{gray}{0.85}

\newtheorem{lemma}{Lemma}
\newtheorem{theorem}{Theorem}

\newtheorem{assume}{Assumption}

\newcommand{\rom}[1]{\uppercase\expandafter{\romannumeral #1\relax}}

\newlength{\dhatheight}

\allowdisplaybreaks    
\definecolor{iccvblue}{rgb}{0.21,0.49,0.74}

\definecolor{colAsym}{RGB}{9, 58, 181}
\definecolor{colSym}{RGB}{173,26,85}

\crefname{algocf}{alg.}{algs.}
\Crefname{algocf}{Algorithm}{Algorithms}

\newacronym{ttp}{TTP}{Tubular Topology Prior}
\newacronym{stp}{STP}{Spherical Topology Prior}
\newacronym{wtp}{WTP}{Weak Topology Prior}
\newacronym{sft}{S\textit{f}T}{Shape-from-Template}
\newacronym{nrsfm}{NRS\textit{f}M}{Non-Rigid Structure-from-Motion}
\newacronym{mdh}{MDH}{Maximum Depth Heuristic}
\newacronym{mdv}{MDV}{Maximum Depth Value}
\newacronym{cnn}{CNN}{Convolutional Neural Network}
\newacronym{dcnn}{DCNN}{Deep Convolutional Neural Network}
\newacronym{crf}{CRF}{Conditional Random Fields}
\newacronym{nng}{NNG}{Nearest-Neighborhood Graph}
\newacronym{tps}{TPS}{Thin-Plate Spline}
\newacronym{sd}{SD}{Standard-Deviation}
\newacronym{apss}{APSS}{Algebraic Point-Set Surfaces}
\newacronym{slam}{SL\textit{a}M}{Simultaneous Localisation and Mapping}
\newacronym{icp}{ICP}{Iterative Closest Point}
\newacronym{wtas}{GWT}{Generic Weak Template}
\newacronym{wtts}{SWT}{Specific Weak Template}
\newacronym{m1}{NRS\textit{f}M++}{NRSfM-suffix-`++'}
\newacronym{m2}{NRS\textit{f}M+}{NRSfM-suffix-`+'}
\newacronym{svd}{SVD}{Singular Value Decomposition}
\newacronym{lls}{LLS}{Linear Least-Squares}
\newacronym{ps}{PS}{Parametric Surface}
\newacronym{nps}{NS}{Non-parametric Surface}
\newacronym{rbf}{RBF}{Radial Basis Function}
\newacronym{gpm}{GPM}{Generalised Problem of Moments}
\newacronym{nrba}{BA}{Bundle Adjustment} 
\newacronym{lm}{LM}{Levenberg-Marquardt}
\newacronym{bnb}{B\textit{n}B}{Branch-and-Bound}
\newacronym{lft}{LFT}{Linear Fractional Transformations}
\newacronym{sp}{SP}{Stereographic Projection}
\newacronym{cf}{CF}{conformal flattening}
\newacronym{sfm}{S\textit{f}M}{Structure-from-Motion}
\newacronym{socp}{SOCP}{Second-Order Cone Programming}
\newacronym{mlh}{MLH}{Maximum-Leg Heuristics}
\newacronym{sdp}{SDP}{Semi-Definite Programming}
\newacronym{arap}{ARAP}{As-Rigid-As-Possible}
\newacronym{sfc}{S\textit{f}C}{Structure-from-Category}
\newacronym{rmse}{RMSE}{Root Mean Square Error}
\newacronym{p2p}{$\mu$P$_2$P}{Mean Point-to-Plane Error}
\newacronym{med}{$\mathrm{E}_d$}{Mean Euclidean Distance}
\newacronym{gt}{G\textit{t}}{Groundtruth}
\newacronym{dcf}{DCF}{Distance Constraint Function}
\newacronym{isonr}{Iso-NRS\textit{f}M}{Isometric NRS\textit{f}M}
\newacronym{mdhnr}{MDH-NRS\textit{f}M}{Maximum Depth Heuristic NRS\textit{f}M}
\newacronym{ao}{AO}{Alternating Optimisation}
\newacronym{grp}{GRP}{Graph Realisation Problem}
\newacronym{agrp}{AF-GRP}{Anchor-Free Graph Realization Problem}
\newacronym{bysr}{BY-$\mathbf{S}^n_+$-R}{Biswas-Ye Semidefinite Relaxation}
\newacronym{edm}{EDM}{Euclidean Distance Matrix}
\newacronym{gdt}{GDT}{Generalised Distance Tensor}
\newacronym{gm}{GM}{Gram Matrix}
\newacronym{dgm}{DGM}{Depth Gram Matrix}
\newacronym{pgm}{PGM}{Position Gram Matrix}
\newacronym{mt}{MT}{Metric Tensegrity}
\newacronym{lrsb}{LRSB}{Low Rank Shape Basis}
\newacronym{sos}{S\textit{o}S}{Sum-of-Squares}
\newacronym{ecp}{EMCP}{Euclidean-distance Matrix Completion Problem}
\newacronym{pcp}{AMCP}{Approximate Metric Completion Problem}
\newacronym{apcp}{SMCP}{Strict Metric Completion Problem}

\newacronym{lrmc}{LRMC}{Low Rank Matrix Completion}
\newacronym{kkt}{KKT}{Karush–Kuhn–Tucker}

\newacronym{au}{$au$}{Arbitrary Units}

\newacronym{sm}{SM}{Supplementary Materials}
\newacronym{qcqp}{QCQP}{Quadratically Constrained Quadratic Programming}

\newacronym{snr}{\textit{s}I-NRS\textit{f}M}{Strictly Isometric NRS\textit{f}M}
\newacronym{qnr}{\textit{q}I-NRS\textit{f}M}{Quasi Isometric NRS\textit{f}M}
\newacronym{enr}{E-NRS\textit{f}M}{Equiareal NRS\textit{f}M}
\newacronym{hnr}{\textit{q}E-NRS\textit{f}M}{Quasi Equiareal NRS\textit{f}M}
\newacronym{sl}{SL}{Sight-Line}
\newacronym{psd}{PSD}{Positive Semi-definite}
\newacronym{dsl}{D\textit{a}SL}{Depth-Along-Sight-Lines}
\newacronym{pp}{P3}{Position in $\mathbb{R}^3$}
\newacronym{po}{PO}{Polynomial Optimisation}
\newacronym{pn}{PN}{Proximal Neighbors}
\newacronym{lmi}{LMI}{Linear Matrix Inequality}
\newacronym{pop}{P\textit{o}P}{Polynomial Optimisation Problem}
\newacronym{pmi}{PMI}{Polynomial Matrix Inequality}
\newacronym{llr}{L\textit{h}LR}{Lasserre's hierarchy of LMI Relaxations}
\newacronym{isg}{ISG}{Isometric Surface Generator}
\newacronym{qsg}{ESG}{Equiareal Surface Generator}
\newacronym{wc}{WC}{White Cartoon T-shirt}
\newacronym{kp}{KP}{Kinect Paper}
\newacronym{sts}{ST}{Stretched T-shirt}
\newacronym{mh}{MH}{Model House}
\newacronym{sb}{SB}{Squeezed Balloon}
\newacronym{slg}{SL}{Stretched Leggings}
\newacronym{b39}{Bramante39M}{Bramante 39 million}

\newacronym{nmr}{NMR}{Nuclear Magnetic Resonance}
\newacronym{snl}{WSNL}{Wireless Sensor Network Localisation}

\newacronym{lut}{L\textit{u}T}{Lookup Table}

\newacronym{mod}{MOD}{Maximum Opposite Depth}
\newacronym{smod}{S\textit{q}MOD}{Squared Maximum Opposite Depth}

\newacronym{s1}{\textit{s}I-$\delta$}{Strictly Isometric - D\textit{a}SL}
\newacronym{s2}{\textit{s}I-P}{Strictly Isometric - \textit{p}3D}
\newacronym{q1}{\textit{m}I-$\delta$}{Maximal Isometric - D\textit{a}SL}
\newacronym{q2}{\textit{m}I-P}{Maximal Isometric - \textit{p}3D}
\newacronym{h1}{\textit{e}I-$\delta$}{Isometric + Equiareal - D\textit{a}SL}
\newacronym{h2}{\textit{e}I-P}{Isometric + Equiareal - \textit{p}3D}

\newacronym{wt}{WT}{Weak Template}
\newacronym{st}{ST}{Strong Template}
\newacronym{gwt}{GWT}{Generic Weak Template}
\newacronym{swt}{SWT}{Specific Weak Template}
\newacronym{qpbo}{QPBO}{Quadratic Pseudo-Boolean Optimisation}
\newacronym{qclp}{QCLP}{Quadratically Constrained Linear Program}

\newacronym{asft}{alg-S\textit{f}T}{algebraic-S\textit{f}T}
\newacronym{anrsfm}{alg-NRS\textit{f}M}{algebraic-NRS\textit{f}M}
\newacronym{crnk}{$c$-rank}{rank up to precision $c$}

\newacronym{od}{OD}{Opposite Depth}
\newacronym{lodh}{LODH}{Least Opposite Depth Heuristics}

\newacronym{anr}{alg-NRS\textit{f}M}{algebraic-NRS\textit{f}M}

\newacronym{mih}{MIH}{Maximum Isometry Heuristics}
\newacronym{pnp}{P\textit{n}P}{Perspective-$n$-Point}

\newacronym{bsl_NonLinLie}{N\textit{l}R}{Non-linear Refinement}
\newacronym{bsl_NonLinICP}{N\textit{l}-P\textit{a}R}{Non-linear Project-and-Refine}
\newacronym{bsl_GMS}{M\textit{s}-GO}{Multi start - Global Optimization}

\newacronym{lc}{L-\textit{c}}{Lipschitz-continuous}
\newacronym{pal}{PARS}{Projected Area Response Surface}
\newacronym{pearl}{PEARS}{Projected Elliptical Aspect Response Surface}

\newacronym{pfs}{P\textit{f}S}{Pose-from-Silhouette}
\newacronym{ar}{AR}{Augmented Reality}

\newacronym{aos}{A\textit{o}S}{Area-of-Silhouettes}

\newacronym{pd}{PD}{Phlegmatic Dragon}
\newacronym{stbu}{SB}{Stanford Bunny}
\newacronym{pelb}{PB}{Pelvic Bone}

\newacronym{gp}{GlOptiPoS}{Globally Optimal P\textit{f}S}
\newacronym{gpp}{GlOptiPoS+}{GlOptiPoS~+~Non-linear Refinement}

\newacronym{gpersp}{GlOptiPoS$_{\Pi}$}{GlOptiPoS Perspective}
\newacronym{gppersp}{GlOptiPoS$_{\Pi}$+}{GlOptiPoS+ Perspective}
\newacronym{gppersp8}{GlOptiPoS$_{\Pi}\pm8$}{GlOptiPoS+ Perspective $\pm8$}

\newacronym{oe}{OE}{Orientation Error}
\newacronym{te}{TE}{Translational Error}

\newacronym{ld}{LD\textit{o}BB}{Largest Diagonal of Bounding-Box}

\newacronym{tipfs}{STI-Pose}{Silhouette Texture Independent Pose}

\newacronym{tipfsa}{STI-Pose-A}{STI-Pose-A}
\newacronym{tipfsb}{STI-Pose-B}{STI-Pose-B}

\newacronym{tipfsorth}{STI-Pose$_{\Pi_{\mathrm{O}}}$}{STI-Pose (orthographic)}

\newacronym{pso}{PSO}{Particle Swarm Optimisation}

\newacronym{bcot}{B\textit{c}OT}{\textit{B}ino\textit{c}ular \textit{O}bject \textit{T}racking}

\newacronym{dac}{DAC}{Deep Active Contours}

\newacronym{p1e}{P$_1$E}{Perspective-1-Ellipsoid}

\newacronym{dof}{D\textit{o}F}{Degrees-of-Freedom}

\newacronym{sfs}{S\textit{f}S}{Shape-from-Silhouette}

\begin{document}
\maketitle
\begin{abstract}
We solve the problem of determining the pose of known shapes in $\mathbb{R}^3$ from their unoccluded silhouettes. The pose is determined up to global optimality using a simple yet under-explored property of the area-of-silhouette: its continuity w.r.t trajectories in the rotation space. The proposed method utilises pre-computed silhouette-signatures, modelled as a response surface of the area-of-silhouettes. Querying this silhouette-signature response surface for pose estimation leads to a strong branching of the rotation search space, making resolution-guided candidate search feasible. Additionally, we utilise the aspect ratio of 2D ellipses fitted to projected silhouettes as an auxiliary global shape signature to accelerate the pose search. This combined strategy forms the first method to efficiently estimate globally optimal pose from just the silhouettes, without being guided by correspondences, for any shape, irrespective of its convexity and genus. We validate our method on synthetic and real examples, demonstrating significantly improved accuracy against comparable approaches.\footnote{\textbf{Code and data:} \href{https://agnivsen.github.io/pose-from-silhouette/}{agnivsen.github.io/pose-from-silhouette/}}
\end{abstract}
    
\section{Introduction}\label{sec_intro}

Estimating the global pose of 3D objects from a single image typically relies on point correspondences between the object template and the image. We address this pose-estimation problem using a much weaker cue -- the object's projected silhouette; we call this the \y{pfs} problem. Our approach discards all assumptions related to the object's shape, such as convexity or genus, and solve this \y{pfs} problem without additional information such as point correspondences or direct image intensities. Our input requirements are the projective silhouettes and a template of the object (in standard formats, e.g., triangulated mesh or semi-dense point cloud). Our solution is globally optimal when the shape-silhouette combination admits a unique solution. However, for symmetric shapes, global solutions for \y{pfs} may not be unique; in such cases, we estimate exactly one of these many redundant solutions -- a reasonable outcome from silhouettes alone. This work presents the first globally optimal solution to the \y{pfs} problem.   

\noindent \textbf{Motivation}.~Solving \y{pfs} is appealing due to breadth of its motivating use-cases. Silhouettes are a key visual cue for diverse engineering applications, e.g.: autonomous driving \citep{wang2020directshape}, robotic manipulation \citep{hebert2012combined}, robotic navigation \citep{albanis2020dronepose}, \y{ar} \citep{chen2025robust}, surgical \y{ar} \citep{collins2020augmented}, space navigation \citep{guo2021pose}, reconstruction from sparse views in medical imaging \citep{sengupta2025shape}, automated flight control \citep{zhang2023mc}, and industrial quality control \citep{perez2023optimal}. A closer inspection of these usages reveal silhouettes to be always used in conjunction with additional visual cues, be it feature correspondences, image intensities, or temporal priors -- an unsurprising finding. \y{pfs} is indeed a fundamentally ill-posed problem; finding its global optima is challenging with no existing solutions. 

\noindent\textbf{Contributions.}~We solve the problem of \textit{pose-estimation using just the unoccluded silhouettes of rigid objects} as input data. Our approach for solving this \y{pfs} problem involves pre-computing global shape signature responses for a given object-template and using this pre-computed shape-signature to infer pose from any arbitrary silhouette of this object. Specifically, we contribute the following:
\begin{itemize}[leftmargin=0.5em]
    \item \textbf{Globally optimal silhouette-based pose estimation.}~Building upon the framework of \citet{hartley2009global}, we introduce shape signatures of orthographic silhouettes that vary continuously with an object's orientation, enabling non-trivial branching strategies that lead to solutions converging to global optima, up to discretisation.

    \item \textbf{Geometric shape signatures and refinement.}~We propose two intuitive yet powerful silhouette descriptors: the \y{aos} and the aspect ratio of ellipses fitted to the silhouettes; we show that they capture sufficient geometric information for accurate pose estimation.  Orientation is recovered numerically, followed by a post-hoc non-linear refinement on the $\mathbb{SE}(3)$ manifold.

    \item \textbf{Generalisation to perspective imaging.}~Assuming depth priors as input, the same formulation achieves \emph{near-optimal accuracy} for perspective silhouettes, demonstrating strong practical performance while preserving an explainable, initialisation-free design.
\end{itemize}


\noindent We validate our approach on many synthetic and real datasets, strongly outperforming compared approaches.
\section{Background}\label{sec_background}

Pose estimation from 2D silhouettes has been widely studied when considered in conjunction with 3D-to-2D point correspondences \citep{kong2017using, sengupta2025shape} or with alternative correspondence cues such as contour generators \citep{menudetmodel}, image textures \citep{toppe2011silhouette, chen2025robust, cui2024silhouette}, or both \citep{prasad2005fast}. However, correspondence-free pose from just the 2D silhouettes/contours has not been solved up to global optimality for general shapes. There exist specialised approaches for specific shapes, such as the case of pose-estimation from ellipse-ellipsoid matches \citep{gaudilliere2023perspective} (which can be considered as pose estimation from elliptical silhouettes), for surfaces of revolution \citep{zhang2009using}, and for cylindrical objects \citep{gummeson2024relative}. None of these approaches, however, can be generalised to arbitrary shapes. There exist some older methods that utilise the notion of `silhouette-lookup' \citep{howe2004silhouette,howe2007silhouette} to estimate poses of deformable structures from learned silhouette appearances along with some other methods \citep{vijayakumar1998invariant, lazebnik2002pencils} that recognize objects based on the pencil of viewing tangent planes along the silhouette from smooth objects, but no efforts have been made to solve the pose estimation problem to global optimality, even in the rigid case. Recently, deep-learning based approaches \citep{wang2023deep} have been proposed for the \y{pfs} problem; unfortunately, they happen to be local methods necessitating an initial pose for their estimation process and requires colour information along object boundary. A \y{pso} based approach has recently been proposed for perspective silhouettes \citep{cui2024silhouette}, though it remains dependent on approximate depth bounds and remains stochastic with no optimality guarantees. \textbf{No globally optimal method exists for pose estimation from silhouettes alone.}
\section{Method}\label{sec_method}

Our methodological description begins by explaining the problem setup (\cref{sec_prblm_setup}) and problem statement (\cref{sec_prblm_stmt}) followed by our strategy of branching the rotation space with global shape signatures (\cref{sec_rot_brnch}) which culminates in a globally optimal and efficient pose estimation approach (\cref{sec_pose_est}). 

\subsection{Problem Setup}\label{sec_prblm_setup}
We assume the availability of any standard shape prior (e.g.: triangulated/tetrahedral mesh, explicit/implicit surfaces, etc.) of the object that is to be localized by its silhouettes. A dense point cloud of 3D surface points is obtained from this shape prior via standard methods (e.g.: sampling \citep{corsini2012efficient} for meshes). We denote this point cloud as $\mathbf{Q} = [\mathbf{P}_1, \hdots, \mathbf{P}_M] \in \mathbb{R}^{3 \times M}$ maintaining $M$ at a sufficiently large value to ensure accurate extraction of silhouettes from projections of $\mathbf{Q}$. We denote by $\Pi_{\mathrm{O}}(\mathbf{Q})$, the point cloud orthographically projected into the $XY$-plane, obtained by simply eliminating the third row of $\mathbf{Q}$. Silhouettes are represented discretely by the ordered sequence $\mathbf{G} = \langle \mathbf{s}_k \in \mathbb{R}^2 | k \in \mathbb{Z} \rangle$ of 2D points within the $XY$-plane. There exists a function $S(\cdot)$ that acts on a 2D point cloud computing the outer boundary of this point cloud with standard methods \citep{edelsbrunner1983shape}, i.e., $\mathbf{G} = S\big( \Pi_{\mathrm{O}}(\mathbf{Q}) \big)$ if $\mathbf{G}$ is the orthographic silhouette of $\mathbf{Q}$. The input silhouette is given by $\mathbf{G}^{\ast}$

\subsection{Problem Statement}\label{sec_prblm_stmt}

The posed problem is defined in terms of the Hausdorff metric. Given an input silhouette $\mathbf{G}^{\ast}$ and an orthographic silhouette $\mathbf{G}$ of the template $\mathbf{Q}$, the Hausdorff distance between the two silhouettes is:
\begin{equation}\label{eq_org_haussdorff_prob}
    \begin{gathered}
        H(\mathbf{G},\mathbf{G}^{\ast}) = \mathrm{max}\Big( \max_{\mathbf{s}_k \in \mathbf{G}} \big(\min_{\mathbf{s}_k^{\ast} \in \mathbf{G}^{\ast}} \|\mathbf{s}_k - \mathbf{s}_k^{\ast}\|\big) ,\\ \max_{\mathbf{s}_k^{\ast} \in \mathbf{G}^{\ast}} \big(\min_{\mathbf{s}_k \in \mathbf{G}} \|\mathbf{s}_k^{\ast} - \mathbf{s}_k\|\big) \Big).
    \end{gathered}
\end{equation}
\noindent However, the orthographic silhouette of any model point cloud $\mathbf{Q}$ that has been rotated by $\mathbf{R} \in \mathbb{SO}(3)$ and translated by $(\mathbf{t}^{\top}, 0)^{\top}$ for some $\mathbf{t} \in \mathbb{R}^2$ is $\tilde{S}(\mathbf{Q},\mathbf{R},\mathbf{t}) = S\big( \Pi_{\mathrm{O}}(\mathbf{R}\mathbf{Q} + (\mathbf{t}^{\top}, 0)^{\top})\big)$. Therefore, the \y{pfs} problem we want to solve relates to a combination of \cref{eq_org_haussdorff_prob} and $\tilde{S}$ as follows:
\begin{equation}\label{key_problem}
    \begin{gathered}
        \boxed{\min_{\mathbf{R} \in \mathbb{SO}(3), \mathbf{t} \in \mathbf{R}^2} H(\tilde{\mathbf{G}},\mathbf{G}^{\ast}), \quad \text{s.t.:} \quad \tilde{\mathbf{G}} = \tilde{S}(\mathbf{Q},\mathbf{R},\mathbf{t})}.
    \end{gathered}
\end{equation}
\noindent Solving \cref{key_problem} is challenging due to non-convexity of the $\mathbb{SO}(3)$ manifold and the Hausdorff distance $H(\tilde{\mathbf{G}},\mathbf{G}^{\ast})$, exacerbated by the possible non-uniqueness of globally-optimal solutions. Importantly, we do not make any assumptions about shape symmetry, convexity, or genus for our proposed method. However, we do make a mild assumption about the solution space of \cref{key_problem}:
\begin{assume}\label{assum_sol_exists}
    There exists at least one pair of roto-translations $(\mathbf{R}_{\mathrm{opt}}, \mathbf{t}_{\mathrm{opt}})$ such that the Hausdorff distance $H(\tilde{\mathbf{G}},\mathbf{G}^{\ast}) \sim 0$ for $\tilde{\mathbf{G}} = \tilde{S}(\mathbf{Q},\mathbf{R}_{\mathrm{opt}},\mathbf{t}_{\mathrm{opt}})$.
\end{assume}
\noindent \Cref{assum_sol_exists} is a practically reasonable emphasis on presence of at least one solution in the solution space. 

\noindent\textbf{Input shape requirement.}~A necessary condition: projected silhouette of $\mathbf{Q}$ admits well-defined area for all $\mathbf{R} \in \mathbb{SO}(3)$, ruling out degenerate projections collapsing to lower-dimensional sets, e.g.: points, lines, or curves.

\subsection{Rotation Space Branching}\label{sec_rot_brnch}
A natural approach for solving \cref{key_problem} is to identify a small subset of $\mathbb{SO}(3)$ that encompasses its feasible set. Our proposed method identifies such subsets by leveraging global shape signatures of silhouettes, irrespective of their non-uniqueness, to partition the $\mathbb{SO}(3)$ search space.

\subsubsection{Area of silhouettes as shape signature}
We use the area of a region enclosed by the silhouette as a global geometric feature, to subdivide the search space. We introduce the function $A$ that acts on a silhouette to compute the area enclosed by it. This area can be regarded as a signature of the shape that is given by the silhouette. We begin with some simple observations about this shape signature:
\begin{theorem}\label{cont_theorem}
$A(\tilde{\mathbf{G}})$ is \y{lc} w.r.t. any sequence of rotations of $\mathbf{Q}$ in $\mathbb{R}^3$ as long as the sequence of rotations is \y{lc} and $\mathbf{Q}$ is representable with a finite collection of triangles (e.g., a triangulated mesh).
\end{theorem}
\begin{proof}  See section~1 of supplementary.
\end{proof}
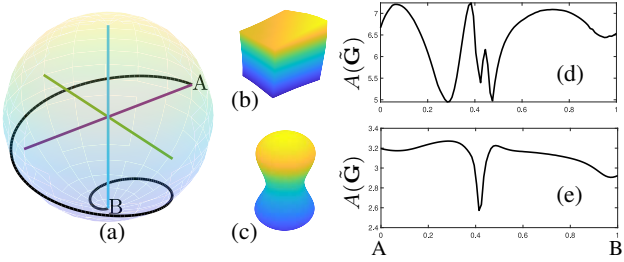
\begin{figure}
    \begin{overpic}[width=0.47\textwidth]{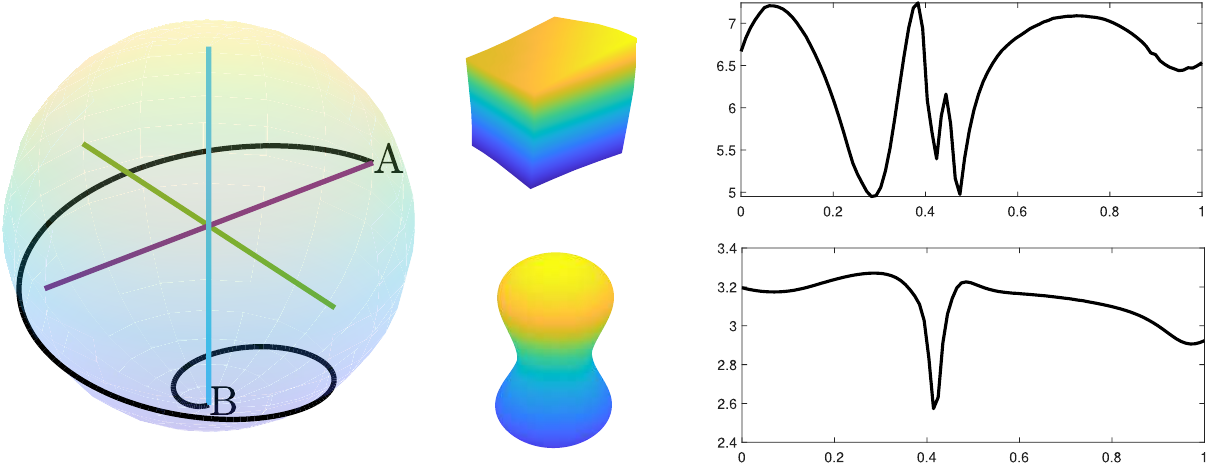}
    \put(37,0){\small (a)}
    \put(87,0){\small (c)}
    \put(87,50){\small (b)}
    \put(210,60){\small (d)}
    \put(210,14){\small (e)}
    \put(140,-6){\footnotesize A}
    \put(230,-6){\footnotesize B}
    \put(128,11){\footnotesize \rotatebox{90}{$A(\tilde{\mathbf{G}})$}}
    \put(128,56){\footnotesize \rotatebox{90}{$A(\tilde{\mathbf{G}})$}}
  \end{overpic}
  \vspace{9pt}
  \caption{When a \y{lc} trajectory from A to B on an unit sphere (a) is mapped to $\mathbb{SO}(3)$ and applied to two arbitrary shapes (b,c), the resulting evolution of orthographic \y{aos}, as shown in (d,e) with its abscissa mapped as $[A,B] \mapsto [0,1]$, is \y{lc} }
\label{fig_illust_1}
\end{figure}

\noindent As a consequence of \cref{cont_theorem}, $A$ is differentiable `almost everywhere' in the compact subset of $\mathbb{R}^9$ embedding $\mathbb{SO}(3)$ due to Rademacher's theorem, and more importantly, have bounded gradients. Although silhouette boundary computation via numeric methods \citep{edelsbrunner1983shape} from discrete representations (e.g.: dense point cloud) for complicated shapes may introduce artefacts in the computed \y{aos}, this does not appear to be a problem in practical cases; we offer two examples of the \y{lc}-ness of computed \y{aos} in arbitrary shapes, shown in \cref{fig_illust_1}. Thus, if the translation $\mathbf{t}$ in \cref{key_problem} is accounted for and if there exists a map $\vartheta:\mathbb{SO}(3)\mapsto\mathbb{R}$ mapping every point of the $\mathbb{SO}(3)$ manifold to the corresponding area of $\tilde{\mathbf{G}}$, then the intersection of the input silhouette's area $A(\mathbf{G}^{\ast})$ with the continuous surface of all possible values of $A(\tilde{\mathbf{G}}),$ must contain the global optima, since $H(\tilde{\mathbf{G}},\mathbf{G}^{\ast}) \sim 0$ implies $|A(\mathbf{G}^{\ast}) - A(\tilde{\mathbf{G}})| \sim 0$. Nonetheless, a conventional \y{bnb} search over $\mathbb{SO}(3)$ (e.g.: \citep{hartley2009global}) is a very expensive proposition. \Cref{sec_ssr} demonstrates a strategy for significant reduction in search space.

\subsubsection{Search space reduction}\label{sec_ssr}
Our goal now is to find the subspace of $\mathbb{SO}(3) \times \mathbb{R}^2$ where $|A(\mathbf{G}^{\ast}) - A(\tilde{\mathbf{G}})| \sim 0$ is satisfied. However, the translational component $\mathbf{t} \in \mathbb{R}^2$ is uninteresting, since: \rom{1}) it can be determined by simply comparing the centroids of modelled and input silhouettes, and \rom{2}) more importantly, for any silhouette $\tilde{\mathbf{G}}$ modelled as $\tilde{\mathbf{G}} = \tilde{S}(\mathbf{Q},\mathbf{R},\mathbf{t})$, the area function $A(\tilde{\mathbf{G}})$ is invariant w.r.t $\mathbf{t} \in \mathbb{R}^2$ and rotation along $Z$-axis. This motivates solving $\mathbf{R} \in \mathbb{SO}(3)$ independently of $\mathbf{t} \in \mathbb{R}^2$. Use of area-based shape signature causes $\mathbb{SO}(3)$ to further split into two categories: a) rotation along $X$ and $Y$ axes ($\mathfrak{R}_{XY}$) that causes $A(\tilde{\mathbf{G}})$ to vary smoothly and b) rotation along $Z$ axis ($\mathfrak{R}_{Z}$) that leaves $A(\tilde{\mathbf{G}})$ invariant (projection is on the $XY$-plane). Thus, the area-based search can be restricted to $\mathfrak{R}_{XY}$, while $\mathfrak{R}_{Z}$ requires a different signature (given later in \cref{sec_pose_est}). Trivially deriving $\mathbf{R}$ from Euler angles is problematic due to non-commutativity of rotations in $\mathfrak{R}_{XY}$ and $\mathfrak{R}_{Z}$, thus necessitating a re-parametrisation.

\noindent \textbf{Postel projection}.~The `azimuthal-equidistant' projection, more commonly known as Postel projection, is a map from a rotation of magnitude $\alpha$ along some unit vector $\hat{\mathbf{v}}$ to a point $\alpha\hat{\mathbf{v}} \in [-\pi, \pi]^3 \subset \mathbb{R}^3$ \citep{hartley2009global, Wikipedia_2024}. The Postel map is typically represented as a map via the quaternion space, given as $\big( \cos{(\frac{\alpha}{2})}, ~\sin{(\frac{\alpha}{2})}\hat{\mathbf{v}}^{\top} \big)^{\top} \mapsto \alpha\hat{\mathbf{v}}$. But the entire $[-\pi, \pi]^3$ space is redundant as well; there exists a sphere of radius $\pi$ centred at origin, we term this sphere as the Postel ball $\mathcal{S}_{\pi}$, and every point in $[-\pi, \pi]^3$ outside $\mathcal{S}_{\pi}$ represents a rotation duplicated with some other point inside $\mathcal{S}_{\pi}$. We represent by $F:\mathbb{R}^3 \mapsto \mathbb{SO}(3)$, a function that takes any point inside $\mathcal{S}_{\pi}$ (except its centre) to its equivalent rotation matrix in $\mathbb{SO}(3)$ by a function $F\big((\alpha, \mathbf{v}^{\top})^{\top}\big)$, detailed in section 6 of the supplementary.



\noindent \textbf{Rotation sufficiency over a disc.}~The search for some specific \y{aos} can therefore be confined to $\mathcal{S}_{\pi}$ without loss of generality. For some $(\alpha, \mathbf{v}) \in \mathcal{S}_{\pi}$, the \y{aos} is obtained by computing $\mathbf{R} \in \mathbb{SO}(3)$ using $F(\cdot)$, and applying this rotation to $\mathbf{Q}$ before computing $A\big( \tilde{S}(\mathbf{Q}, \mathbf{R}, \mathbf{0}_3) \big)$. However, there are still some redundancies in $\mathcal{S}_{\pi}$, since it represents both $\mathfrak{R}_{XY}$ and $\mathfrak{R}_{Z}$, i.e.,:
\begin{lemma}\label{lemm_search_sufficiency}
    Given some $F:(\alpha, \mathbf{v})\mapsto\mathbf{R}$, we have $A\big( \tilde{S}(\mathbf{Q}, \mathbf{R}, \mathbf{0}_3) \big) = A\big( \tilde{S}(\mathbf{Q}, \mathbf{R}_x, \mathbf{0}_3) \big)$ for all $\mathbf{R}_x = F\big((\alpha, \mathbf{v}_x^{\top})^{\top}\big)$, if $\langle \mathbf{v}, (0,0,1)^{\top} \rangle = \langle \mathbf{v}_x, (0,0,1)^{\top} \rangle$.
\end{lemma}
\begin{proof}
See section~2 of supplementary.
\end{proof}
\noindent Thus, instead of sampling the highly redundant $\mathcal{S}_{\pi}$, we sample only the disc of the intersection of $\mathcal{S}_{\pi}$ with its XZ-plane, we call this the Postel disc $\mathcal{D}_{\pi}$ and, for any point $\mathbf{d} \in \mathcal{D}_{\pi} \subset \mathbb{R}^2$, we denote by the function $G(\mathbf{d}) = (\alpha, \mathbf{v}^{\top})^{\top} = (\|\mathbf{d}\|, \mathbf{d}_1/\|\mathbf{d}\|, 0, \mathbf{d}_2/\|\mathbf{d}\|)^{\top}$, the invertible map from $\mathcal{D}_{\pi}$ to the $\mathcal{S}_{\pi}$.

\noindent \textbf{Silhouette area signature.}~Given that the set of all possible projected area values can be pre-computed from just $\mathcal{D}_{\pi}$, we do so by semi-densely sampling $\mathcal{D}_{\pi}$ and recording the projected area profile as a shape signature, we term this signature \y{pal}; which is a non-injective map $\mathcal{A}:\mathcal{D}_{\pi} \mapsto \mathbb{R}$ and surjective to some subset of $\mathbb{R}$. Pre-computation of \y{pal} can be done as given in algorithm-1 of supplementary. An example of \y{pal} for the triangulated 3D model of \y{pd} \citep{phlegm} is shown in \cref{fig_illust_2}.

\begin{figure*}[t]
\centering
    \begin{overpic}[width=0.9\textwidth]{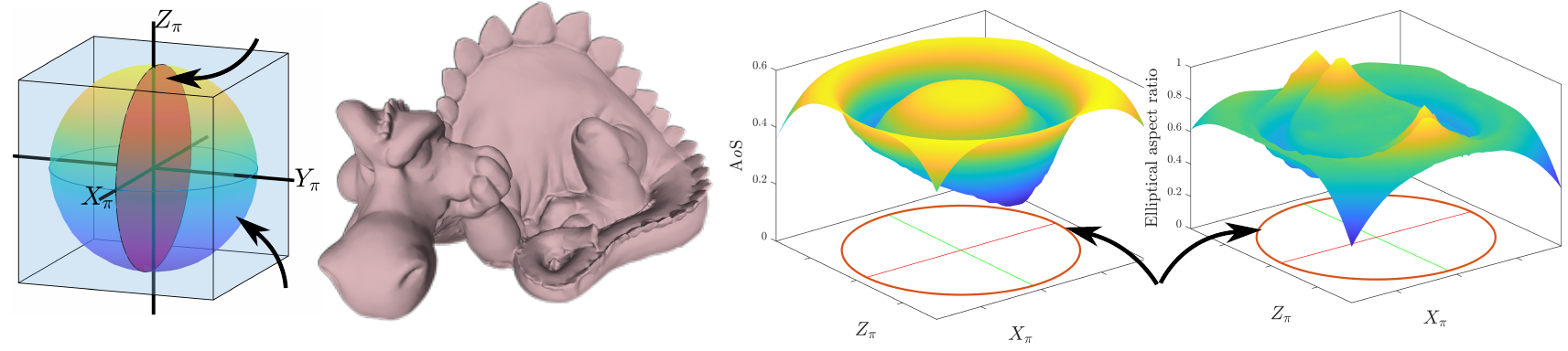}
    \put(40,0){\small (a)}
    \put(150,0){\small (b)}
    \put(275,0){\small (c)}
    \put(395,0){\small (d)}
    \put(80,10){\small $\mathcal{S}_{\pi}$}
    \put(73,89){\small $\mathcal{D}_{\pi}$}
    \put(327,11){\small $\mathcal{D}_{\pi}$}
    \put(270,100){\small $\mathcal{A}$}
    \put(390,100){\small $\mathscr{E}$}
  \end{overpic}
  \vspace{7pt}
  \caption{(a) Shows the Postel ball ($\mathcal{S}_{\pi}$) and disc ($\mathcal{D}_{\pi}$) inside a cube of length $\pi$ with axes $(X_{\pi}, Y_{\pi}, Z_{\pi})$, (b) shows the 3D template of \y{pd}, (c) shows the mapping of $\mathcal{D}_{\pi}$ to \y{aos} for \y{pd}, and (d) shows the mapping of $\mathcal{D}_{\pi}$ to elliptical aspect ratio for \y{pd}}
\label{fig_illust_2}
\end{figure*}
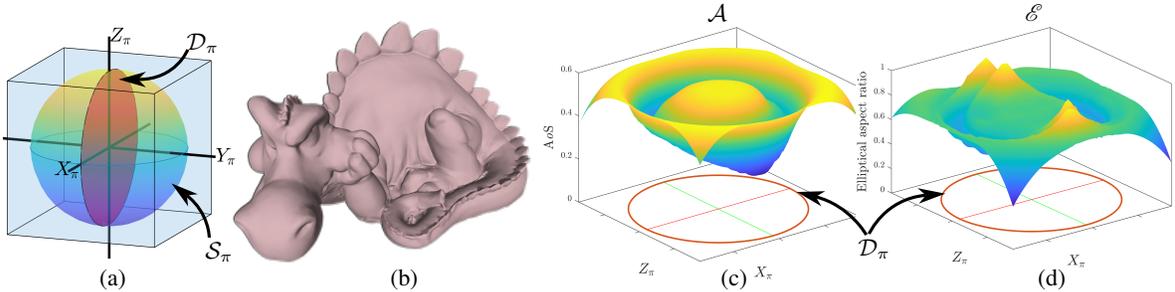

\subsection{Pose Estimation}\label{sec_pose_est}
We now show how to utilise pre-computed \y{aos} signatures to estimate pose up to global optimality.

\subsubsection{Globally optimal pose from area of silhouettes}
With pre-computed \y{pal}, pose inference from a silhouette involves: a) intersecting the input \y{aos} with \y{pal}, b) identifying `candidate' rotations for pose estimation, and c) efficiently searching among them for the optimal pose.

\noindent \textbf{Area-\y{pal} intersection}.~First the translation between $\mathbf{G}^{\ast}$ and $\tilde{\mathbf{G}}$ is computed as the difference of centroids $\mathbf{t} = C(\tilde{S}(\mathbf{Q},\mathbf{I}_3,\mathbf{0})) - C(\mathbf{G}^{\ast})$, where $C(\cdot)$ computes the centroid of the 2D points. Next, the iso-contour of the intersection of $A(\mathbf{G}^{\ast})$ with the map $\mathcal{A}$ is determined by standard contouring operations, namely, by first upgrading the semi-dense $\mathcal{A}$ to a continuous surface and then computing the intersection with $A(\mathbf{G}^{\ast})$ using marching squares \citep{matlab_func}, giving us $U_{\mathcal{A}} = \{\mathbf{d}_j \in \mathcal{D}_{\pi}, ~j \in [1, N_D]\}$ discrete set of intersecting points, such that:
\begin{equation}\label{eq_criter_1}
    |A\bigg(\tilde{S}\Big(\mathbf{Q},F\big(G(\mathbf{d}_j)\big), \mathbf{t}\Big)\bigg) - A(\mathbf{G}^{\ast})| \leq \epsilon_{xy},
\end{equation}
\noindent where $\epsilon_{xy}$ is a tunable threshold.

\noindent\textbf{Candidate solution search}. Since $\{\mathbf{d}j\}$ is derived from $\mathcal{D}_{\pi}$, the solution set spans $\mathfrak{R}_{XY}$ but not $\mathfrak{R}_{Z}$, allowing arbitrary $Z$-axis rotations without affecting projected \y{aos} (\cref{lemm_search_sufficiency}). Evaluating \y{aos} alone in 2D provides no further template pose validation. Aligning $\tilde{\mathbf{G}}$ and $\mathbf{G}^{\ast}$ via dominant singular values is a potential approach, but computationally expensive for large $N_D$. Instead, we leverage the 1D projections of $\tilde{\mathbf{G}}$ and $\mathbf{G}^{\ast}$ along $X$ and $Y$, where their lengths, $L_x(\tilde{\mathbf{G}})$ and $L_y(\tilde{\mathbf{G}})$, provide useful additional information.
We thus seek a rotation matrix $\mathbf{R}_z$ such that, with $\mathbf{R}_c = \mathbf{R}_zF\big(G(\mathbf{d}_j)\big)$, we have:
\begin{equation}\label{cond_accept_2D}
\begin{gathered}
|L_x\bigg(\tilde{S}\Big(\mathbf{Q},\mathbf{R}_c,\mathbf{t}\Big)\bigg) - L_x(\mathbf{G}^{\ast})| \leq \epsilon_z ~\wedge \\
|L_y\bigg(\tilde{S}\Big(\mathbf{Q},\mathbf{R}_c,\mathbf{t}\Big)\bigg) - L_y(\mathbf{G}^{\ast})| \leq \epsilon_z,
\end{gathered} 
\end{equation}
\noindent where $\epsilon_z$ is a threshold. With uniform samples along Z-axis $\theta_{z,k} \in U(0, 2\pi)$, we accept a pose as a `candidate' for a good solution if substituting $\mathbf{R}_z = M(\theta_{z,k}), ~\forall k \in [1, N_Z]$ in \cref{cond_accept_2D}, for some $N_Z$, satisfies both of its conditions; $M(\theta_{z,k})$ gives the rotation matrix from the Euler angle $\theta_{z,k}$ along Z-axis. Thus, for every point $\mathbf{d}_j$, we get a set of rotation matrices:
\begin{equation}\label{eq_c_j}
    C_j = \{M(\theta_{z,k'})F\big(G(\mathbf{d}_j)\big), ~\forall k' \in [1, N'_{Z,j}]\},
\end{equation}
\noindent where every rotation matrix in $C_j$ satisfies \cref{cond_accept_2D}. Thus $\tilde{C} = \bigcup_{j=1}^{N_D} C_j$ gives us the global set of candidate solutions. Thereafter, the filtering of $\tilde{C}$ into actual solutions of \cref{key_problem} is done by exhaustive search in this reduced feasible set $\tilde{C}$. Importantly, we show the existence of $\epsilon_o$-global optimality, meaning a candidate solution exists whose distance to global optima in $\mathbb{SO}(3)$ is bounded by $\epsilon_o$.

\begin{theorem}{($\epsilon_o$-global optimality)}\label{theorem_global_ortho}
There must exist an element of $\tilde{C}$ which lies within a ball of finite radius $\epsilon_o$ in $\mathbb{SO}(3)$ from the global optima of \cref{key_problem} with $\lim_{\epsilon_{xy}, \epsilon_z \rightarrow 0} \epsilon_o = 0$ and $(\epsilon_{xy}, \epsilon_z)$ are free parameters of sampling, controllable to approach zero.
\end{theorem}
\begin{proof}
See section~3 of supplementary.
\end{proof}
\noindent In practice, we maintain an upper bound on $|\tilde{C}| \leq \lambda_c$, randomly eliminating candidates from $\tilde{C}$ if its count exceeds $\lambda_c$, to aid faster convergence; more details in algorithm-2 (supplementary). The guarantee in \cref{theorem_global_ortho} trivially extends to scaled-orthographic projection; the optimality conditions and convexity properties of the orthographic case remain unchanged under global scaling.

\subsubsection{Ellipse aspect-ratio driven acceleration}\label{sec_acc}
$\tilde{C}$, obtained by sampling the Postel ball for equality of projected silhouette area, already strongly branches the search space, yet further branching remains possible - which can be verified by observing that any global shape signature, if \y{lc} w.r.t rotation, can serve to subdivide the Postel ball in a manner analogous to the strategy (with area) in \cref{lemm_search_sufficiency}.  

Given $\tilde{\mathbf{G}}$, we fit an ellipse $\mathscr{E}$ algebraically \citep{bektas2015least}, noting that its aspect ratio $AR_{\mathscr{E}}$, i.e., its major-to-minor axis ratio, is heuristically \y{lc} w.r.t. $\mathbf{R}$ and $\mathbf{t}$ in most cases. Since $AR_{\mathscr{E}}$ is used solely for acceleration, the lack of a formal proof of its \y{lc}-ness w.r.t. $\mathbf{R}$ does not affect global optimality, as shown in algorithm-2 (supplementary).

\noindent \textbf{Silhouette elliptical aspect ratio signature.}~Following the same strategy as algorithm-1 (supplementary), we learn the mapping from $\mathcal{D}_{\pi}$ to $AR_{\mathscr{E}}$ via semi-dense random sampling of $\mathcal{D}_{\pi}$, creating a new shape signature \y{pearl}, which we denote by the non-injective map $\mathscr{E}:\mathcal{D}_{\pi} \mapsto \mathbb{R}$.

\noindent \textbf{Accelerated pose estimation.}~The pose estimation follows two steps: we \textit{first} estimate the contour of the intersection of $AR_{\mathscr{E}}(\mathbf{G}^{\ast})$, denoted as a function computing $AR_{\mathscr{E}}$ of $\mathbf{G}^{\ast}$ with a mild abuse of notation, with $\mathcal{D}_{\pi}$, giving us $U_{\mathscr{E}} = \{\mathbf{h}_{j'} \in \mathcal{D}_{\pi}, ~j' \in [1, N_E]\}$ discrete set of intersecting points such that:
\begin{equation}\label{eq_criter_2_ear}
    |A\bigg(\tilde{S}\Big(\mathbf{Q},F\big(G(\mathbf{h}_{j'})\big), \mathbf{t}\Big)\bigg) - A(\mathbf{G}^{\ast})| \leq \epsilon_{e},
\end{equation}
\noindent where $\epsilon_{e}$ is a tunable threshold. For the \textit{second} step, we find the nearest neighbours between $\{\mathbf{d}_{j}\}$ and $\{\mathbf{h}_{j'}\}$, giving us $U_{\mathcal{A} \cap \mathscr{E}} = \{\tilde{\mathbf{d}}_h \in \mathcal{D}_{\pi}, ~h \in [1, N_{\mathcal{A}\cap\mathscr{E}}]\}$, such that $U_{\mathcal{A} \cap \mathscr{E}} \subseteq U_{\mathcal{A}}$ and :
\begin{equation}\label{eq_intersect_criter}
    \exists j''\in [1, N_E] ~\text{s.t.:}~ \|\tilde{\mathbf{d}}_h - \mathbf{h}_{j''}\| \leq \epsilon_{\cap}, \forall h \in [1, N_{\mathcal{A}\cap\mathscr{E}}],
\end{equation}
\noindent where $\epsilon_{\cap}$ is a parameter denoting an infinitesimal circle around each $\tilde{\mathbf{d}}_h$ which is considered part of the intersection $U_{\mathscr{E}} \cap U_{\mathcal{A}}$. We show an example of the pose estimation outcome with \y{pd} in \cref{fig_illust_3}.

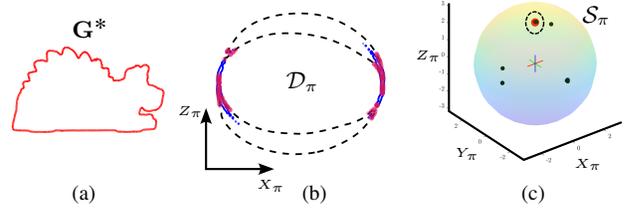
\begin{figure}[t]
\centering
    \begin{overpic}[width=\columnwidth]{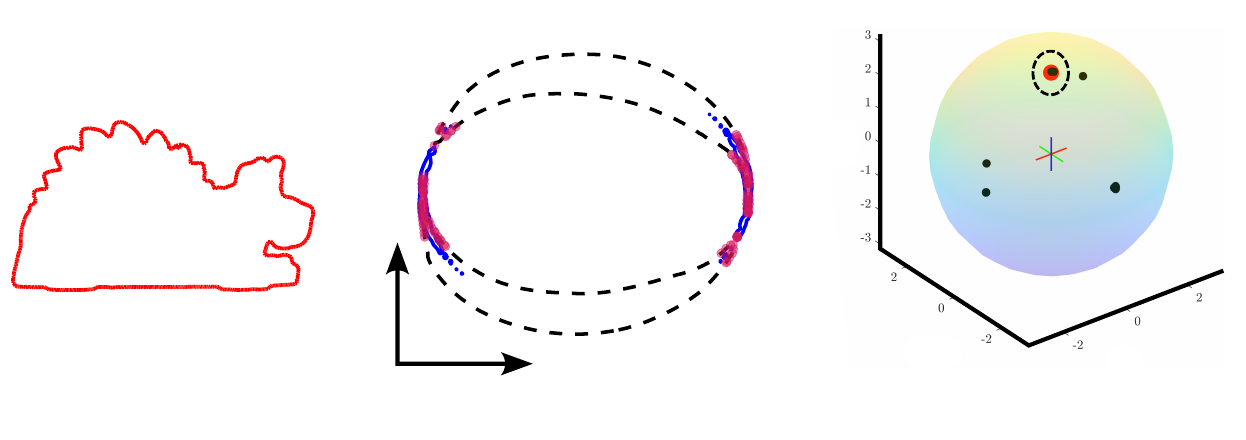}
    \put(25,3){\footnotesize (a)}
    \put(112,3){\footnotesize (b)}
    \put(195,3){\footnotesize (c)}
    \put(25,64){\small $\mathbf{G}^{\ast}$}
    \put(106,45){\small $\mathcal{D}_{\pi}$}
    \put(218,70){\small $\mathcal{S}_{\pi}$}
    \put(65,35){\tiny $Z_{\pi}$}
    \put(95,7){\tiny $X_{\pi}$}
    \put(215,15){\tiny $X_{\pi}$}
    \put(170,19){\tiny $Y_{\pi}$}
    \put(155,56){\tiny $Z_{\pi}$}
  \end{overpic}
  \vspace{1pt}
  \caption{Example of pose estimation: (a) shows input silhouette from \y{pd}, (b) shows $U_{\mathcal{A}}$ in black-dashes, $U_{\mathscr{E}}$ in blue, and $U_{\mathcal{A} \cap \mathscr{E}}$ in red on $\mathcal{D}_{\pi}$, and (c) shows $\tilde{C}$ in black and the \y{gt} pose in red (encircled) inside $\mathcal{S}_{\pi}$; the \y{pal} and \y{pearl} for \y{pd} are given in \cref{fig_illust_2}}
\label{fig_illust_3}
\end{figure}

\noindent \textbf{Resolution pyramid}.~Our final estimation strategy follows a pyramidal approach, initialising search bounds $(\epsilon_{xy}, \epsilon_{z}, \epsilon_{e}, \epsilon_{\cap})$ to satisfy $H(\tilde{\mathbf{G}},\mathbf{G}^{\ast}) \leq \epsilon_{H}$, a desired precision. If violated, typically due to noisy silhouettes, only $(\epsilon_{\cap}, \epsilon_{z})$ are incremented and retried. This iterative algorithm is detailed in algorithm-2 (supplementary).

We denote the pose from algorithm-2 (supplementary) as \y{gp} and algorithm-2 (supplementary) followed by non-linear refinement (\cref{sec_non_lin}) as \y{gpp}. 

\subsection{Non-linear Refinement}\label{sec_non_lin}
The solution pose $(\mathbf{R}_{\mathrm{opt}}, \mathbf{t}_{\mathrm{opt}})$ from algorithm-2 (supplementary) may deviate from the global optimum of \cref{key_problem} due to sampling resolutions $(\epsilon_{xy}, \epsilon_{z}, \epsilon_{\cap}, \epsilon_{H})$ and noise. To mitigate, we apply local, non-linear refinement, optimizing \cref{key_problem} while initializing $(\mathbf{R}, \mathbf{t})$ with $(\mathbf{R}_{\mathrm{opt}}, \mathbf{t}_{\mathrm{opt}})$. The cost is iteratively minimized via steps along the $\mathbb{SE}(3)$ tangent plane, followed by manifold retraction using standard optimization tools \citep{manopt}. The final refined pose is $(\mathbf{R}_{\mathrm{ref}}, \mathbf{t}_{\mathrm{ref}})$.


\subsection{Perspective Approximation}
The optimality guarantee of \cref{theorem_global_ortho} does not extend directly to perspective projection, since perspectivity couples silhouette-based signatures with object translation. Incorporating translation into the $\mathbb{R}^3$ search space would cause exponential growth, rendering \y{bnb}-style global optimization impractical. Hence, following prior work~\citep{cui2024silhouette}, we \textit{assume availability of a coarse depth prior}, obtainable from RGB-D or monocular estimation (e.g.:\citep{Bochkovskii2024:arxiv}). The perspective silhouettes are $\mathbf{G}_{\Pi} = S(\Pi(\mathbf{Q}))$, $\Pi(\cdot)$ being perspective projection, leading to the perspective equivalent of \cref{key_problem}, with $\mathbf{t} \in \mathbb{R}^3$. \Cref{cont_theorem} can be extended to perspective projection, since:
\begin{lemma}
$A(\tilde{\mathbf{G}}_{\Pi})$ is \y{lc} w.r.t. any sequence of rotations of $\mathbf{Q}$ in $\mathbb{R}^3$ as long as the sequence of rotations is \y{lc}, $\mathbf{Q} \notin \{Z = 0\}$, and $\mathbf{Q}$ is representable with a finite collection of triangles
\end{lemma}
\begin{proof} See section 4 of supplementary.
\end{proof}
\noindent \y{pearl} is equivalently heuristically \y{lc} and therefore used for acceleration following \cref{sec_acc}. Given a prior depth, \y{pal} and \y{pearl} are pre-computed offline perspectively at that depth while $\mathbf{G}_{\Pi}^{\ast}$ is transformed to normalised image coordinates, assuming known intrinsics. Non-linear refinement follows \cref{sec_non_lin}, replacing $\mathbf{G}$ with $\mathbf{G}_{\Pi}$. The details of the pipeline are omitted for brevity, obtained as the perspective variants of algorithm 2 (supplementary) given by \y{gpersp} and \y{gpersp} followed by perspective variant of \cref{sec_non_lin} as \y{gppersp}. 

\glsunset{tipfsa}
\glsunset{tipfsb}

\section{Experimental Results}\label{sec_res}
We empirically verify our approach below.

\noindent \textbf{Experimental setup.}~Using three 3D models \y{stbu}, \y{pd}, and \y{pelb} (resp.), visualized in section 7 of supplementary, we use point clouds of 29072, 29120, and 28976 points randomly sampled on their surface to obtain $\mathcal{A}$ and $\mathcal{E}$ via algorithm-1 (supplementary). For pose estimation, independently sampled point clouds are used to ensure a separation between offline pre-computation and pose estimation data. Each experiment applies roto-translation via $\mathbb{SE}(3)$ random-sampling, followed by orthographic projection and silhouette extraction. We also use the 20 objects in the real data from \y{bcot} benchmark dataset~\citep{li2022bcot}, specifically in two modes: \rom{1}) for statistical validation, random poses near the provided \y{gt} are simulated to render silhouettes perspectively (all methods evaluated under identical conditions), \rom{2}) segmented objects from the `\textit{complex\_movable\_handheld}' sequence of \y{bcot} are used to estimate pose. Randomly sampled surfaces of the 20 objects in \y{bcot} are used to learn $\mathcal{A}$ and $\mathcal{E}$ using perspective variant of algorithm-1 (supplementary) at a constant empirical depth prior of 80 \textit{cm} -- all perspective experiments are given without re-learning this depth prior.
\glsunset{bsl_NonLinLie}
\glsunset{bsl_NonLinICP}
\glsunset{bsl_GMS}
\glsunset{tipfsorth}
\begin{table*}[t] \begin{adjustbox}{width=0.9\textwidth,center} \begin{tabular}{crccccccccccccccc}  \toprule \hline\multirow{2}{*}{} & \multirow{2}{*}{} & \multicolumn{3}{c}{\cellcolor[RGB]{200,200,200} \textbf{\y{bsl_NonLinLie}}} & \multicolumn{3}{c}{\cellcolor[RGB]{200,200,200} \textbf{\y{bsl_NonLinICP}}} & \multicolumn{3}{c}{\cellcolor[RGB]{200,200,200} \textbf{\y{bsl_GMS}}} & \multicolumn{3}{c}{\cellcolor[RGB]{200,200,200} \textbf{\y{tipfsorth}}} & \multicolumn{3}{c}{\cellcolor[RGB]{200,200,200} \textbf{\y{gpp}}}\\& & \cellcolor[RGB]{200,200,200} \textbf{RMSE} $\downarrow$  & \cellcolor[RGB]{200,200,200} \textbf{OE} $\downarrow$  & \cellcolor[RGB]{200,200,200} \textbf{TE} $\downarrow$  & \cellcolor[RGB]{200,200,200} \textbf{RMSE} $\downarrow$  & \cellcolor[RGB]{200,200,200} \textbf{OE} $\downarrow$  & \cellcolor[RGB]{200,200,200} \textbf{TE} $\downarrow$  & \cellcolor[RGB]{200,200,200} \textbf{RMSE} $\downarrow$  & \cellcolor[RGB]{200,200,200} \textbf{OE} $\downarrow$  & \cellcolor[RGB]{200,200,200} \textbf{TE} $\downarrow$  & \cellcolor[RGB]{200,200,200} \textbf{RMSE} $\downarrow$  & \cellcolor[RGB]{200,200,200} \textbf{OE} $\downarrow$  & \cellcolor[RGB]{200,200,200} \textbf{TE} $\downarrow$  & \cellcolor[RGB]{200,200,200} \textbf{RMSE} $\downarrow$  & \cellcolor[RGB]{200,200,200} \textbf{OE} $\downarrow$  & \cellcolor[RGB]{200,200,200} \textbf{TE} $\downarrow$ \\ \cline{3-17} \multirow{ 3}{*}{\textbf{\y{stbu}}} & \textbf{Mean}  & 42.35 &54.66 &8.85 &43.13 &55.62 &7.97 &28.89 &33.32 &\textbf{3.87} &\textbf{9.75} &\textbf{3.12} &9.74 &\textbf{\underline{0.46}} &\textbf{\underline{0.32}} &\textbf{\underline{0.14}}\\ & \textbf{SD}  & 8.18 &25.48 &8.85 &\textbf{4.18} &25.2 &6.51 &14.53 &30.64 &\textbf{1.99} &9.07 &\textbf{13.63} &9.06 &\textbf{\underline{0.35}} &\textbf{\underline{0.36}} &\textbf{\underline{0.15}}\\ & \textbf{Max.}  & 64.23 & \textbf{101.03} &46.3 &50.37 & \textcolor{red}{110.55} &27.2 &\textbf{48.46} & \textcolor{red}{104.25} &\textbf{9.91} &67.77 &\textcolor{red}{111.45} &67.76 &\textbf{\underline{1.6}} &\textbf{\underline{2.2}} &\textbf{\underline{0.79}}\\ \cline{2-17} \multirow{ 3}{*}{\textbf{\y{pd}}} & \textbf{Mean}  & 38.22 &44.93 &5.9 &36.77 &40.17 &5.18 &\textbf{32.58} &38.73 &\textbf{4.18} &101.55 &\textbf{4.29} &101.41 &\textbf{\underline{0.91}} &\textbf{\underline{0.61}} &\textbf{\underline{0.32}}\\ & \textbf{SD}  & \textbf{7.85} &23.41 &3.67 &8.96 &26.54 &4.42 &11.88 &31.63 &\textbf{2.29} &248.78 &\textbf{16.9} &248.68 &\textbf{\underline{0.92}} &\textbf{\underline{0.88}} &\textbf{\underline{0.32}}\\ & \textbf{Max.}  & 50.01 &\textbf{94.51} &17.3 &48.4 & \textcolor{red}{118.91} &18.82 &\textbf{48.38} & \textcolor{red}{114.81} &\textbf{10.56} & \textcolor{red}{1879.94} &\textcolor{red}{118.87} & \textcolor{red}{1879.57} &\textbf{\underline{7}} &\textbf{\underline{7.43}} &\textbf{\underline{2.16}}\\ \cline{2-17} \multirow{ 3}{*}{\textbf{\y{pelb}}} & \textbf{Mean}  & 35.86 &41.87 &8.66 &37.72 &39.25 &12.41 &\textbf{32.86} &37.17 &\textbf{5.41} &78.99 &\textbf{3.47} &78.9 &\textbf{\underline{0.76}} &\textbf{\underline{0.5}} &\textbf{\underline{0.26}}\\ & \textbf{SD}  & 11.03 &28.98 &6.54 &\textbf{9.3} &31.44 &7.17 &13.44 &29.21 &\textbf{3.3} &180.81 &\textbf{15.15} &180.7 &\textbf{\underline{0.95}} &\textbf{\underline{0.93}} &\textbf{\underline{0.27}}\\ & \textbf{Max.}  & 50.22 & \textcolor{red}{100.79} &31.7 &\textbf{49.78} & \textcolor{red}{106.18} &25.86 &51.21 &\textbf{97.92} &\textbf{14.92} & \textcolor{red}{1362.52} &105.64 & \textcolor{red}{1362.14} &\textbf{\underline{7.41}} &\textbf{\underline{8.58}} &\textbf{\underline{1.48}}\\\hline\bottomrule \end{tabular} \end{adjustbox} \caption{Accuracy of the methods \y{bsl_NonLinLie}, \y{bsl_NonLinICP}, \y{bsl_GMS}, \y{tipfsorth}, and \y{gpp} on \y{stbu}, \y{pd}, and \y{pelb} (\footnotesize{best values in each category are marked with \underline{\textbf{bold-underlines}}, second-best values are in \textbf{bold}, very large error values have been highlighted in \textcolor{red}{red})}}   \label{tab_label_ToChange} \end{table*} 
\glsreset{bsl_NonLinLie}
\glsreset{bsl_NonLinICP}
\glsreset{bsl_GMS}
\glsreset{tipfsorth}

\noindent \textbf{Metric.}~We use three error metrics: \rom{1}) \y{oe}, the mean absolute Euler angles of the optimal orientation between \y{gt} and estimated posed shape, computed with Horn's method~\citep{horn1988closed}, \rom{2}) \y{te}, computed as $\|\mathbf{t}_{\mathrm{gt}} - \mathbf{t}_{\mathrm{est}}\|_2$, and \rom{3}) \y{rmse}, measuring the difference between $\mathbf{Q}_{\mathrm{gt}}$ and $\mathbf{Q}_{\mathrm{est}}$. Here, the suffix `$\mathrm{gt}$' and `$\mathrm{est}$' denotes the \y{gt} and estimated values (resp.). \y{te} and \y{rmse} are expressed as percentages of \y{ld} for the 3D templates of \y{pd}, \y{stbu}, and \y{pelb} and for \y{bcot}, they are expressed in \textit{mm}.

\noindent \textbf{Compared methods}.~We compare against three variants of \y{pfs} solutions that are intuitive: \rom{1}) \y{bsl_NonLinLie}, a non-convex solver using \y{lm} \citep{more2006levenberg} parametrised on Lie algebra \citep{gilmore2006lie} of the rotation group, \rom{2}) \y{bsl_NonLinICP}, an iterative projection and refinement method aligning $\mathbf{G}^{\ast}$ and $\tilde{\mathbf{G}}$ via Euclidean distance minimisation, and \rom{3}) \y{bsl_GMS} following \citep{ugray2007scatter} while solving \cref{key_problem} directly; \textit{details of these approaches are given in section-7-of-supplementary}. Importantly, \y{bsl_NonLinLie}, \y{bsl_NonLinICP}, and \y{bsl_GMS} are `pose-estimation methods' (and not reconstruction or registration methods). We do extensive comparison against~\citep{cui2024silhouette}, denoted by \y{tipfs}, as our closest recent baseline method. We also adapt \y{tipfs} into an orthographic equivalent \y{tipfsorth} for fair comparison with orthographic silhouettes.

\subsection{Orthographic Silhouettes}\label{subsec_posest}
We present our results on orthographic silhouettes below.

\noindent \textbf{Experiment}.~We validate our approach by estimating \y{pfs} with \y{gpp}, repeated 200 times on randomly posed silhouettes of \y{pd}, \y{stbu}, and \y{pelb}, with additive silhouette noise of \y{sd} $\sim1\%$ of \y{ld}. The compared methods \y{bsl_NonLinLie}, \y{bsl_NonLinICP}, \y{bsl_GMS}, and \y{tipfsorth} undergo identical experiments, with results in \cref{tab_label_ToChange}. 

\noindent \textbf{Comparison}.~\y{tipfsorth} is the second-best method in all experiments. Problems due to its stochastic nature is clear from the maximum error values of \y{tipfsorth}, which are higher than \y{bsl_NonLinLie}, \y{bsl_NonLinICP}, and \y{bsl_GMS} (\y{oe} of $\sim$110$^{\circ}$ is practically unusable). Our method is 89.74\%, 85.78\%, and 85.59\% better than \y{tipfsorth} in mean \y{oe} -- a significant improvement -- while our worst case \y{oe} is $\sim$8.6$^{\circ}$ for \y{pelb}, which is due to numerical artefacts and not typically `catastrophic' for most use-cases. More importantly, mean \y{oe} of our method for all shapes are $\ll1^{\circ}$, confirming the power of global optimality in \y{pfs}. Error histograms of \y{gp} and \y{gpp}, given in section 7 (supplementary), confirm robustness across all experiments. Notably, \y{gp} exhibits higher \y{te} accuracy since its translation is in a closed-form, while \y{gpp} optimises on the $\mathbb{SE}(3)$ manifold, prioritising \y{oe} at minor expense to \y{te}.

\noindent \textbf{Symmetry as proxy for algorithmic complexity}.~A formal derivation of \y{gp}'s algorithmic complexity is intractable, however, the number of candidate solutions $|\tilde{C}|$ intuitively scales with the object’s rotational symmetry, e.g.: a perfect sphere yields theoretically infinite solutions. To validate this intuition, we synthesize real spherical-harmonic shapes, denoted with a mild notational-abuse as $Y_l^m$ following~\citep{wiki:Spherical_harmonics}, with $l\in [2,6] \subset \mathbb{Z}^+$ and $m =l-1$, ordered reverse-chronologically by decreasing symmetry (labels \textit{B}–\textit{F} in \cref{sph_harm}\rom{1}; \textit{A} denotes the sphere). The number of candidates obtained by \y{gp} over 25 trials (\cref{sph_harm}\rom{2}) shows approximately logarithmic decay with diminishing symmetry, while the sphere exhibits exponentially higher $|\tilde{C}|$. This confirms: (\rom{1}) $|\tilde{C}|$ contracts as symmetry reduces, implying higher efficiency for real-world shapes; and (\rom{2}) for highly symmetric objects, silhouette-only pose estimation is inherently ambiguous -- a geometric truism.

\noindent \textbf{Effect of noisy silhouettes on accuracy}.~Evaluating silhouette noise impact on pose estimation is done with \y{pelb}, a genus-1 shape (\cref{tab_label_ToChange}). Three noise regimes are considered: \textit{low} (\y{sd} $\approx 1\%$ of \y{ld}), \textit{medium} (\y{sd} $\approx 2\%$ of \y{ld}), and \textit{high} (\y{sd} $\approx 4\%$ of \y{ld}); exemplar noisy silhouettes are shown in section 7 of supplementary. A pose estimate is deemed `successful' if \y{oe} $\leq 6^{\circ}$ and \y{te} $\leq 2\%$ of \y{ld}. Unlike \cref{tab_label_ToChange}, we analyse not only the best solution but seven best candidate poses from $\tilde{C}$ of \y{gp}, corresponding to the lowest $H(\mathbf{G},\mathbf{G}^{\ast})$ values (\cref{eq_org_haussdorff_prob}). Results in \cref{fig:images_noisy} gives success percentages. Increasing noise shifts optimal candidates away from \y{gt}. For \textit{low} and \textit{medium} noise, top-ranked candidate is successful in 100\% of trials (confirming \cref{tab_label_ToChange}). \textit{High} noise shifts candidates deeper in hierarchy, sometimes beyond the first seven levels. The behaviour is consistent with graceful noise degradation: with sufficient candidate sampling (algorithm-2, supplementary), valid poses are eventually recovered.

\noindent \textbf{Impact of parameters on runtime and accuracy}.~We analyse \y{gp} varying thresholds for filtering candidate poses, detecting intersections, and adjusting template point-count for signature computation, focusing on \y{pd} as the toughest case (\cref{tab_label_ToChange}). \Cref{fig:ablate}a show \y{rmse} and time changes with $\epsilon_z \in [0, 0.15]$ in 7 steps; accuracy remains stable, but outliers vanish at $\epsilon_z = 0.15$, albeit with increased runtime. \Cref{fig:ablate}b reveal a V-shaped \y{rmse} response to $\epsilon_{\cap} \in [0, 0.15]$, minimizing at $\epsilon_{\cap} \sim 0.08$ due to maintaining $\lambda_c \sim 10^2$. \Cref{fig:ablate}c indicate accuracy rises with increasing $P \in [100, 29121]$, though for greater runtime -- an expected trade-off. Our implementation is in MATLAB executed on Intel$^{\text{\textregistered}}$ Core™ i9-10920X (24-core) CPU, 128 GB RAM. Importantly, such runtimes are typical of \y{bnb} methods and remains parallelisable as a future work.

\begin{figure}[b]
  \centering



  \begin{subfigure}[b]{0.6\linewidth}
    \centering
    \includegraphics[width=\linewidth,keepaspectratio]{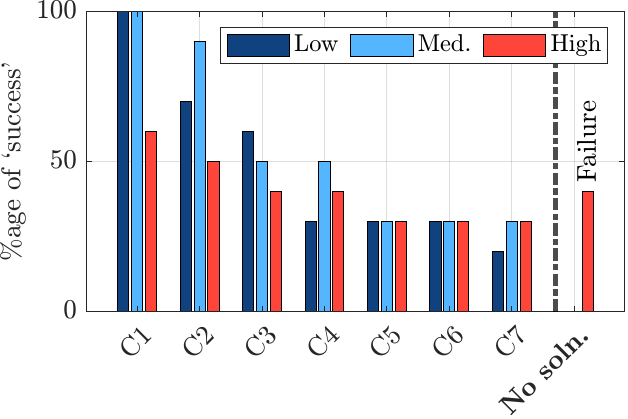}
  \end{subfigure}

  \caption{Success rate of pose estimation across the seven best solution candidates, denoted C$x$ ($x \in [1,7]$), including some failure cases only for `high' noise}
  \label{fig:images_noisy}
\end{figure}

\begin{figure}[b]
  \centering
  \begin{overpic}[width=0.9\linewidth]{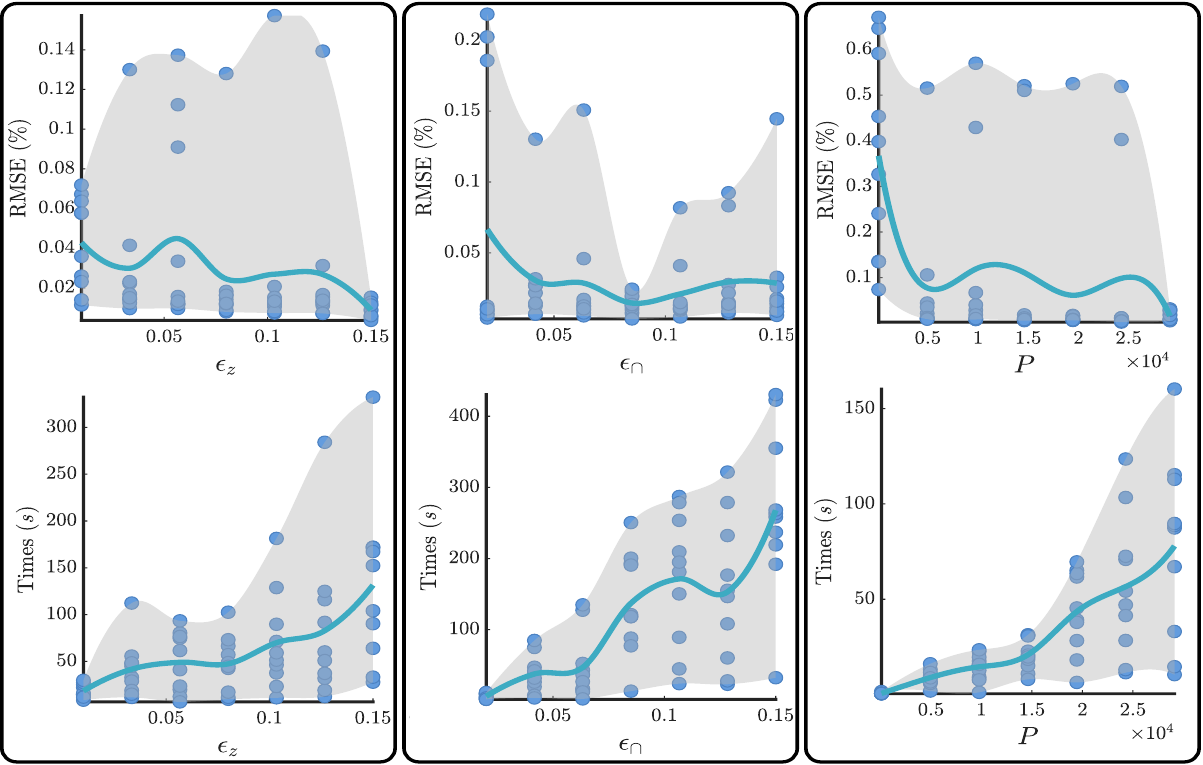}
    \put(3,70){\small (a)}
    \put(74,70){\small (b)}
    \put(146,70){\small (c)}
  \end{overpic}
  
  \vspace{2mm}
  \caption{\y{rmse} (top-row) and runtime (bottom-row) of \y{gp} under varying (a) $\epsilon_z$, (b) $\epsilon_{\cap}$, and (c) $P$. Blue curves denote interpolated mean-curves \footnotesize{(please zoom-in)}.}
  \label{fig:ablate}
\end{figure}


\glsunset{tipfsa}
\glsunset{tipfsb}
\glsunset{gppersp}
\glsunset{gppersp8}
\begin{table}[t] \begin{adjustbox}{width=0.85\columnwidth,center} \begin{tabular}{rcccccccc}  \toprule \hline\multirow{2}{*}{} & \multicolumn{2}{c}{\cellcolor[RGB]{200,200,200} \textbf{\y{tipfsa}}} & \multicolumn{2}{c}{\cellcolor[RGB]{200,200,200} \textbf{\y{tipfsb}}} & \multicolumn{2}{c}{\cellcolor[RGB]{200,200,200} \textbf{\y{gppersp8}}} & \multicolumn{2}{c}{\cellcolor[RGB]{200,200,200} \textbf{\y{gppersp}}}\\ & \cellcolor[RGB]{200,200,200} \textbf{Mean} & \cellcolor[RGB]{200,200,200} \textbf{SD} & \cellcolor[RGB]{200,200,200} \textbf{Mean} & \cellcolor[RGB]{200,200,200} \textbf{SD} & \cellcolor[RGB]{200,200,200} \textbf{Mean} & \cellcolor[RGB]{200,200,200} \textbf{SD} & \cellcolor[RGB]{200,200,200} \textbf{Mean} & \cellcolor[RGB]{200,200,200} \textbf{SD}\\ \cline{2-9} \textbf{\textcolor{colAsym}{Ape}}  & 32.74 &28.07 &\textbf{10.01} &\textbf{\underline{1.55}} &54.61 &28.99 &\textbf{\underline{1.6}} &\textbf{2.25}\\\textbf{\textcolor{colAsym}{Cat}}  & 31.71 &21.98 &\textbf{19.29} &\textbf{9.87} &78.8 &10.88 &\textbf{\underline{0.72}} &\textbf{\underline{0.46}}\\\textbf{\textcolor{colAsym}{Jack}}  & 32.04 &35.14 &10.71 &19.05 &\textbf{10.04} &\textbf{4.01} &\textbf{\underline{1.75}} &\textbf{\underline{2.22}}\\\textbf{\textcolor{colAsym}{Squirrel}}  & \textbf{36.2} &42.85 &38.08 &\textbf{42.64} &121.91 &79.19 &\textbf{\underline{2.1}} &\textbf{\underline{2.14}}\\\textbf{\textcolor{colAsym}{Stitch}}  & 29.07 &33.34 &18.87 &27.74 &\textbf{8.26} &\textbf{1.14} &\textbf{\underline{0.71}} &\textbf{\underline{0.41}}\\\textbf{\textcolor{colAsym}{Vampire queen}}  & 49.53 &32.61 &\textbf{26.98} &\textbf{6.77} &57.6 &8.07 &\textbf{\underline{2.24}} &\textbf{\underline{1.76}}\\\textbf{\textcolor{colSym}{3D Touch}}  & 48.61 &31.62 &33.4 &19.99 &\textbf{21.21} &\textbf{\underline{7.44}} &\textbf{\underline{3.99}} &\textbf{12.21}\\\textbf{\textcolor{colSym}{Auto GPS}}  & 93.44 &47.81 &88.39 &41.51 &\textbf{52.24} &\textbf{27.05} &\textbf{\underline{1.57}} &\textbf{\underline{2.05}}\\\textbf{\textcolor{colSym}{Bracket}}  & 47.35 &50.1 &\textbf{15.06} &\textbf{\underline{1.79}} &17.41 &\textbf{11} &\textbf{\underline{10.26}} &27.87\\\textbf{\textcolor{colSym}{Deadpool}}  & 40.74 &36.63 &\textbf{15.08} &\textbf{\underline{11.74}} &67.12 &38.58 &\textbf{\underline{4}} &\textbf{12.27}\\\textbf{\textcolor{colSym}{Driller}}  & 76.93 &61.11 &62.74 &51.7 &\textbf{8.22} &\textbf{\underline{1.21}} &\textbf{\underline{1.31}} &\textbf{1.36}\\\textbf{\textcolor{colSym}{Flashlight}}  & 57.25 &42.57 &\textbf{42.51} &33.59 &88.76 &\textbf{9.31} &\textbf{\underline{1.37}} &\textbf{\underline{1.35}}\\\textbf{\textcolor{colSym}{Lamp clamp}}  & \textbf{73.2} &\textbf{49.22} &77.03 &\textbf{\underline{44.52}} &113.25 &61.66 &\textbf{\underline{36.55}} &50.63\\\textbf{\textcolor{colSym}{Lego}}  & 58.35 &28.78 &\textbf{29.5} &\textbf{25.1} &47.3 &32.09 &\textbf{\underline{19.28}} &\textbf{\underline{23.47}}\\\textbf{\textcolor{colSym}{RJ45}}  & 47.62 &48.01 &\textbf{32.22} &38.76 &94.57 &\textbf{24.81} &\textbf{\underline{5.39}} &\textbf{\underline{19.25}}\\\textbf{\textcolor{colSym}{RTI Arm}}  & 110.85 &59.98 &79.31 &66.97 &\textbf{41.53} &\textbf{\underline{27.78}} &\textbf{\underline{32.37}} &\textbf{43.8}\\\textbf{\textcolor{colSym}{Standtube}}  & 36.68 &42.23 &\textbf{29.05} &\textbf{34.91} &109.56 &65.78 &\textbf{\underline{1.08}} &\textbf{\underline{0.71}}\\\textbf{\textcolor{colSym}{Teapot}}  & 46.92 &40.43 &\textbf{25.41} &30.98 &81.4 &\textbf{\underline{10.12}} &\textbf{\underline{10.51}} &\textbf{24.57}\\\textbf{\textcolor{colSym}{Tube}}  & 47.84 &40.6 &27.96 &29.2 &\textbf{21.71} &\textbf{3.6} &\textbf{\underline{1.36}} &\textbf{\underline{0.82}}\\\textbf{\textcolor{colSym}{Wall Shelf}}  & 48.32 &39.88 &37.3 &34.58 &\textbf{10.24} &\textbf{2.13} &\textbf{\underline{0.76}} &\textbf{\underline{0.72}}\\\hline\bottomrule \end{tabular} \end{adjustbox} \caption{\y{rmse}($\downarrow$) in \textit{mm}: mean and \y{sd} for \y{tipfsa}, \y{tipfsb}, \y{gppersp}, and \y{gppersp8} for the 20 \y{bcot} objects. Asymmetric and symmetric objects have been colour-coded \textcolor{colAsym}{\textbf{blue}} and \textcolor{colSym}{\textbf{red}} respectively (\footnotesize{best values in each category are marked with \underline{\textbf{bold-underlines}}, second-best values are in \textbf{bold}})}   \label{tab_bcot_sti_rmse} \end{table}

\begin{table}[t] \begin{adjustbox}{width=0.85\columnwidth,center} \begin{tabular}{rcccccccc}  \toprule \hline\multirow{2}{*}{} & \multicolumn{2}{c}{\cellcolor[RGB]{200,200,200} \textbf{\y{tipfsa}}} & \multicolumn{2}{c}{\cellcolor[RGB]{200,200,200} \textbf{\y{tipfsb}}} & \multicolumn{2}{c}{\cellcolor[RGB]{200,200,200} \textbf{\y{gppersp8}}} & \multicolumn{2}{c}{\cellcolor[RGB]{200,200,200} \textbf{\y{gppersp}}}\\ & \cellcolor[RGB]{200,200,200} \textbf{Mean} & \cellcolor[RGB]{200,200,200} \textbf{SD} & \cellcolor[RGB]{200,200,200} \textbf{Mean} & \cellcolor[RGB]{200,200,200} \textbf{SD} & \cellcolor[RGB]{200,200,200} \textbf{Mean} & \cellcolor[RGB]{200,200,200} \textbf{SD} & \cellcolor[RGB]{200,200,200} \textbf{Mean} & \cellcolor[RGB]{200,200,200} \textbf{SD}\\ \cline{2-9} \textbf{\textcolor{colAsym}{Ape}}  & 8.94 &16.93 &\textbf{0.86} &\textbf{1.04} &4.7 &2.86 &\textbf{\underline{0.29}} &\textbf{\underline{0.33}}\\\textbf{\textcolor{colAsym}{Cat}}  & 3.12 &\textbf{4.74} &\textbf{3.07} &9.44 &37.54 &25.67 &\textbf{\underline{0.18}} &\textbf{\underline{0.15}}\\\textbf{\textcolor{colAsym}{Jack}}  & 7.21 &12.83 &6.57 &22.55 &\textbf{\underline{0.37}} &\textbf{\underline{0.15}} &\textbf{0.38} &\textbf{0.44}\\\textbf{\textcolor{colAsym}{Squirrel}}  & 9.67 &25.22 &10.32 &15.46 &\textbf{0.75} &\textbf{\underline{0.43}} &\textbf{\underline{0.5}} &\textbf{0.55}\\\textbf{\textcolor{colAsym}{Stitch}}  & 10.09 &23.98 &10 &20.24 &\textbf{1.13} &\textbf{0.39} &\textbf{\underline{0.21}} &\textbf{\underline{0.2}}\\\textbf{\textcolor{colAsym}{Vampire queen}}  & 14.62 &25.45 &4.82 &10.99 &\textbf{1.96} &\textbf{0.57} &\textbf{\underline{0.55}} &\textbf{\underline{0.56}}\\\textbf{\textcolor{colSym}{3D Touch}}  & 14.75 &27.09 &8.14 &14.38 &\textbf{\underline{1.6}} &\textbf{\underline{0.85}} &\textbf{3.31} &\textbf{13.36}\\\textbf{\textcolor{colSym}{Auto GPS}}  & 24.93 &35.16 &25.73 &34.36 &\textbf{3.28} &\textbf{1.74} &\textbf{\underline{0.4}} &\textbf{\underline{0.58}}\\\textbf{\textcolor{colSym}{Bracket}}  & 15 &26.46 &\textbf{\underline{0.63}} &\textbf{1.03} &\textbf{0.76} &\textbf{\underline{0.47}} &6.06 &17.84\\\textbf{\textcolor{colSym}{Deadpool}}  & 9.28 &18.2 &\textbf{\underline{1.08}} &\textbf{3.6} &\textbf{1.4} &\textbf{\underline{0.65}} &1.44 &5.13\\\textbf{\textcolor{colSym}{Driller}}  & 33.64 &38.54 &15.5 &25.23 &\textbf{0.3} &\textbf{\underline{0.04}} &\textbf{\underline{0.13}} &\textbf{0.11}\\\textbf{\textcolor{colSym}{Flashlight}}  & 32.87 &40.72 &22.88 &34.23 &\textbf{3.89} &\textbf{1.53} &\textbf{\underline{0.31}} &\textbf{\underline{0.32}}\\\textbf{\textcolor{colSym}{Lamp clamp}}  & 27.11 &33.53 &\textbf{23.62} &\textbf{25.93} &23.94 &\textbf{\underline{16.22}} &\textbf{\underline{19.18}} &29.86\\\textbf{\textcolor{colSym}{Lego}}  & 47.21 &35.36 &\textbf{17.06} &\textbf{24.89} &60.26 &41.75 &\textbf{\underline{3.25}} &\textbf{\underline{13.67}}\\\textbf{\textcolor{colSym}{RJ45}}  & \textbf{15.6} &25.66 &17.44 &35.28 &45.76 &\textbf{15.27} &\textbf{\underline{2.42}} &\textbf{\underline{9.58}}\\\textbf{\textcolor{colSym}{RTI Arm}}  & 48.17 &35.63 &29.89 &33.35 &\textbf{\underline{0.22}} &\textbf{\underline{0.04}} &\textbf{21.65} &\textbf{29.94}\\\textbf{\textcolor{colSym}{Standtube}}  & 20.09 &34.56 &\textbf{16.42} &\textbf{28.48} &45.78 &31.31 &\textbf{\underline{0.29}} &\textbf{\underline{0.28}}\\\textbf{\textcolor{colSym}{Teapot}}  & 22.65 &27.2 &\textbf{11.91} &\textbf{25.58} &59.19 &36.22 &\textbf{\underline{6.02}} &\textbf{\underline{17.34}}\\\textbf{\textcolor{colSym}{Tube}}  & 13.43 &23.03 &14.23 &24.13 &\textbf{2.4} &\textbf{0.62} &\textbf{\underline{0.29}} &\textbf{\underline{0.2}}\\\textbf{\textcolor{colSym}{Wall Shelf}}  & 25.43 &37.86 &17.47 &25.08 &\textbf{0.31} &\textbf{\underline{0.03}} &\textbf{\underline{0.16}} &\textbf{0.13}\\\hline\bottomrule \end{tabular} \end{adjustbox} \caption{\y{oe}($\downarrow$) in degrees, following same scheme as \cref{tab_bcot_sti_rmse}}   \label{tab_bcot_sti_oe} \end{table}

\begin{table}[t] \begin{adjustbox}{width=0.85\columnwidth,center} \begin{tabular}{rcccccccc}  \toprule \hline\multirow{2}{*}{} & \multicolumn{2}{c}{\cellcolor[RGB]{200,200,200} \textbf{\y{tipfsa}}} & \multicolumn{2}{c}{\cellcolor[RGB]{200,200,200} \textbf{\y{tipfsb}}} & \multicolumn{2}{c}{\cellcolor[RGB]{200,200,200} \textbf{\y{gppersp8}}} & \multicolumn{2}{c}{\cellcolor[RGB]{200,200,200} \textbf{\y{gppersp}}}\\ & \cellcolor[RGB]{200,200,200} \textbf{Mean} & \cellcolor[RGB]{200,200,200} \textbf{SD} & \cellcolor[RGB]{200,200,200} \textbf{Mean} & \cellcolor[RGB]{200,200,200} \textbf{SD} & \cellcolor[RGB]{200,200,200} \textbf{Mean} & \cellcolor[RGB]{200,200,200} \textbf{SD} & \cellcolor[RGB]{200,200,200} \textbf{Mean} & \cellcolor[RGB]{200,200,200} \textbf{SD}\\ \cline{2-9} \textbf{\textcolor{colAsym}{Ape}}  & 25.39 &21.67 &\textbf{9.68} &1.41 &11.94 &\textbf{0.65} &\textbf{\underline{0.25}} &\textbf{\underline{0.22}}\\\textbf{\textcolor{colAsym}{Cat}}  & 29.28 &18.71 &\textbf{17.72} &\textbf{4.57} &26.67 &6.04 &\textbf{\underline{0.16}} &\textbf{\underline{0.13}}\\\textbf{\textcolor{colAsym}{Jack}}  & 26.58 &27.6 &5.81 &1.77 &\textbf{4.17} &\textbf{\underline{0.15}} &\textbf{\underline{0.34}} &\textbf{0.52}\\\textbf{\textcolor{colAsym}{Squirrel}}  & 25.62 &21.39 &\textbf{14.72} &\textbf{8.22} &18.98 &8.56 &\textbf{\underline{0.23}} &\textbf{\underline{0.24}}\\\textbf{\textcolor{colAsym}{Stitch}}  & 23.34 &25.36 &7.26 &9.49 &\textbf{2.87} &\textbf{0.06} &\textbf{\underline{0.03}} &\textbf{\underline{0.02}}\\\textbf{\textcolor{colAsym}{Vampire queen}}  & 38.64 &20.87 &\textbf{24.41} &\textbf{0.88} &48.31 &4.28 &\textbf{\underline{0.59}} &\textbf{\underline{0.59}}\\\textbf{\textcolor{colSym}{3D Touch}}  & 36.45 &19.41 &27.51 &11.56 &\textbf{8.93} &\textbf{\underline{0.12}} &\textbf{\underline{0.39}} &\textbf{0.89}\\\textbf{\textcolor{colSym}{Auto GPS}}  & 64.29 &20.3 &64.39 &22.01 &\textbf{12.31} &\textbf{\underline{0.74}} &\textbf{\underline{0.58}} &\textbf{0.78}\\\textbf{\textcolor{colSym}{Bracket}}  & 36.28 &40.46 &\textbf{14.87} &\textbf{\underline{1.59}} &16.28 &10.21 &\textbf{\underline{2.24}} &\textbf{6.65}\\\textbf{\textcolor{colSym}{Deadpool}}  & 31.68 &26.02 &12.68 &\textbf{\underline{1.21}} &\textbf{7.52} &\textbf{1.55} &\textbf{\underline{0.81}} &2.57\\\textbf{\textcolor{colSym}{Driller}}  & 32.47 &24.58 &21.85 &12.96 &\textbf{8.08} &\textbf{1.28} &\textbf{\underline{0.21}} &\textbf{\underline{0.21}}\\\textbf{\textcolor{colSym}{Flashlight}}  & 34.63 &22.59 &22.09 &8.5 &\textbf{11.46} &\textbf{3.74} &\textbf{\underline{0.14}} &\textbf{\underline{0.12}}\\\textbf{\textcolor{colSym}{Lamp clamp}}  & 29.71 &23.51 &\textbf{11.87} &\textbf{2.32} &30.34 &5.34 &\textbf{\underline{0.45}} &\textbf{\underline{0.28}}\\\textbf{\textcolor{colSym}{Lego}}  & 33.42 &28.99 &8.28 &5.28 &\textbf{0.23} &\textbf{0.14} &\textbf{\underline{0.02}} &\textbf{\underline{0.02}}\\\textbf{\textcolor{colSym}{RJ45}}  & 28.98 &28.04 &\textbf{13.13} &6.04 &16.09 &\textbf{4.23} &\textbf{\underline{0.72}} &\textbf{\underline{2.13}}\\\textbf{\textcolor{colSym}{RTI Arm}}  & 37.98 &43.27 &4.8 &2.63 &\textbf{2.24} &\textbf{\underline{0.04}} &\textbf{\underline{0.22}} &\textbf{0.69}\\\textbf{\textcolor{colSym}{Standtube}}  & 27.05 &29.22 &\textbf{13.51} &13.28 &15.31 &\textbf{0.23} &\textbf{\underline{0.21}} &\textbf{\underline{0.14}}\\\textbf{\textcolor{colSym}{Teapot}}  & 24.38 &23.73 &\textbf{6.02} &\textbf{\underline{1.2}} &18.83 &8.1 &\textbf{\underline{0.91}} &\textbf{1.74}\\\textbf{\textcolor{colSym}{Tube}}  & 32.38 &28.02 &17.04 &11.22 &\textbf{10.06} &\textbf{3.99} &\textbf{\underline{0.19}} &\textbf{\underline{0.15}}\\\textbf{\textcolor{colSym}{Wall Shelf}}  & 34.64 &28.77 &20.44 &18.25 &\textbf{9.38} &\textbf{3.46} &\textbf{\underline{0.16}} &\textbf{\underline{0.16}}\\\hline\bottomrule \end{tabular} \end{adjustbox} \caption{\y{te}($\downarrow$) in \textit{mm}, following same scheme as \cref{tab_bcot_sti_rmse}}   \label{tab_bcot_sti_te} \end{table}

\glsunset{gppersp8}
\subsection{Perspective Silhouettes}
We present our results on perspective silhouettes below.

\noindent \textbf{Experiment}.~For \y{gppersp}, two depth-prior settings are used: \rom{1}) the nominal pre-computation depth for $\mathcal{A}$ and $\mathcal{E}$, and \rom{2}) a perturbed setting (\y{gppersp8}) with random $\pm$8 \textit{cm} deviation from learned signatures, exceeding typical depth-sensor noise. For \y{tipfs}, authors’ provided depth bounds ($[5,100]$ \textit{cm}) yielded unstable results; hence, bounds were improved to provide meaningful results: within $\pm$ 10 \textit{cm} and $\pm$ 5 \textit{cm} of \y{gt} depth, denoted as \y{tipfsa} and \y{tipfsb} (resp.).
The results are summarised in \cref{tab_bcot_sti_rmse,tab_bcot_sti_oe,tab_bcot_sti_te}. The \y{bcot} objects were grouped into 6 symmetric and 14 asymmetric objects, colour-coded in \cref{tab_bcot_sti_rmse,tab_bcot_sti_oe,tab_bcot_sti_te}; sample qualitative results on asymmetric shapes from \textit{complex\_movable\_handheld} sequence are shown in \cref{bcot_qual}, with extensive qualitative results in section 7 (supplementary).

\noindent \textbf{Comparison}.~The accuracy of \y{gppersp} is very high for asymmetric shapes, dropping for symmetric shapes owing to to multiple ambiguous solutions -- consistent with \y{tipfsa}, \y{tipfsb} and \y{gppersp8}. \y{gppersp} is the best method across all shapes while the second position is split between \y{tipfsb} and \y{gppersp8} for \y{rmse} and \y{te}. For \y{oe}, however, \y{gppersp8} is ahead of \y{tipfsb}, confirming its advantage for orientation estimation. \y{tipfsa} lags behind significantly on all metrics, hinting at higher sensitivity of \y{tipfs} to inaccurate depth bounds.

\noindent \textbf{Effect of depth prior on accuracy}.~We vary the deviation of pre-computed and pose-estimation depth from 0 to 70 \textit{mm} in 10 steps and record $H(\mathbf{G}_{\Pi},\mathbf{G}_{\Pi}^{\ast})$ along with \y{rmse}, \y{oe}, and \y{te} for an asymmetric and symmetric shape from \y{bcot}: \textit{Ape} and \textit{AutoGPS} (resp.) -- plots in section 7 (supplementary). We observe: \rom{1}) \textit{AutoGPS} is less accurate than \textit{Ape} as expected, and \rom{2}) for \textit{Ape}, increasing deviation of pre-computed and pose-estimation depth results in outlying errors which could be large, however, the median accuracy remains low. 

\noindent \textbf{Partial comparisons - non-uniform baselines}.~For completeness, brief illustrative comparisons with \y{dac}~\citep{wang2023deep} and \y{p1e}~\citep{gaudilliere2023perspective} are included in section 7 of supplementary, acknowledging their distinctly different and incomparable problem assumptions.

\noindent \textbf{Additional analysis}.~We offer examples of ambiguities induced by symmetric shapes, variance of Hausdorff distance and \y{aos} with point cloud sampling density and silhouette noise, and experiment on thin shell objects from the Bramante39M~\citep{sengupta2025convex, bartoli2025camera} dataset in section 7 (supplementary).

\begin{figure}[h!]
  \centering{
  \begin{subfigure}[t]{0.48\columnwidth}
    \centering
    \includegraphics[height=3.4cm]{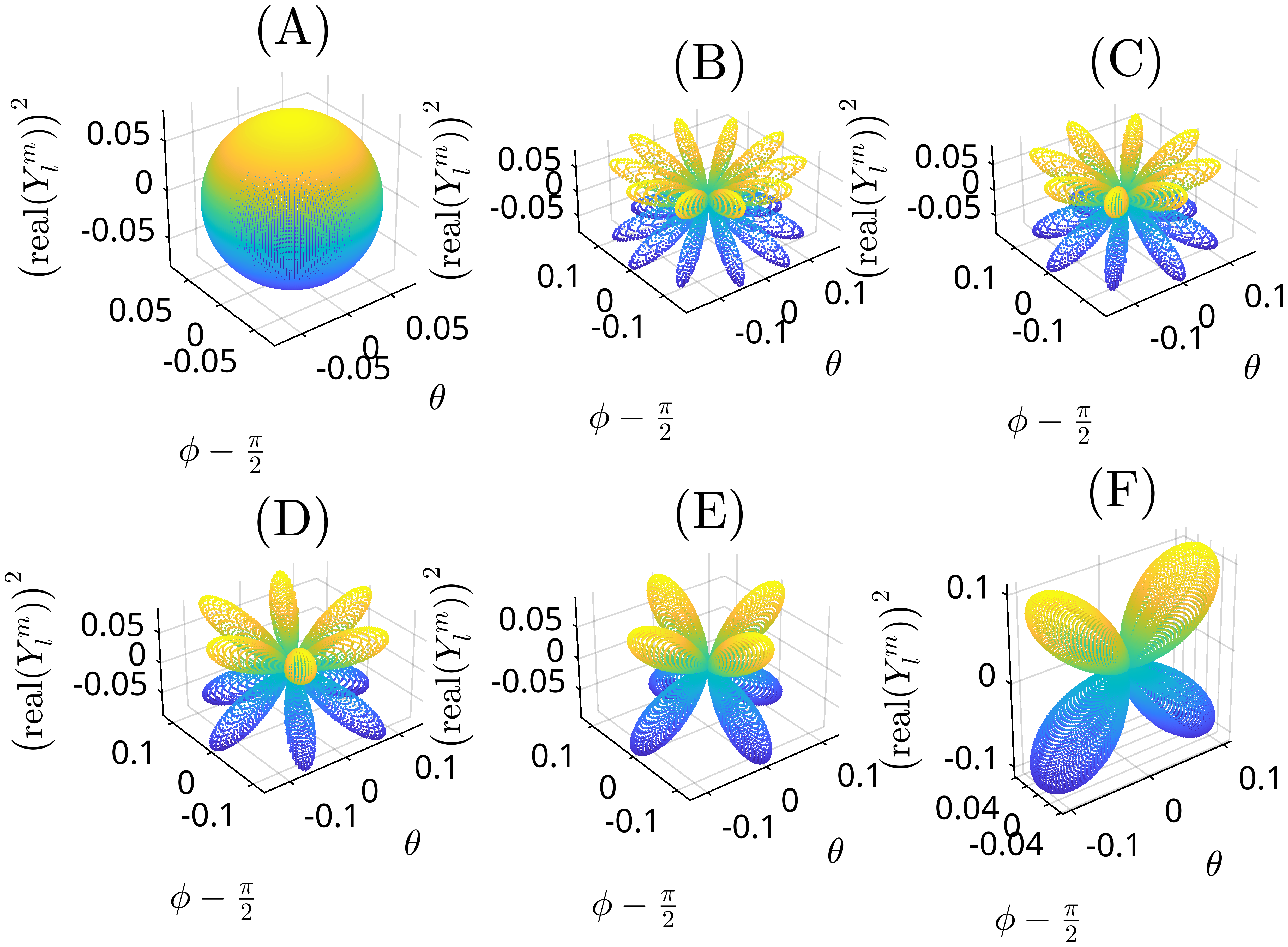}
    \caption*{(\rom{1})}
  \end{subfigure}%
  \hspace{2.5mm}
  \begin{subfigure}[t]{0.48\columnwidth}
    \centering
    \includegraphics[height=3.3cm]{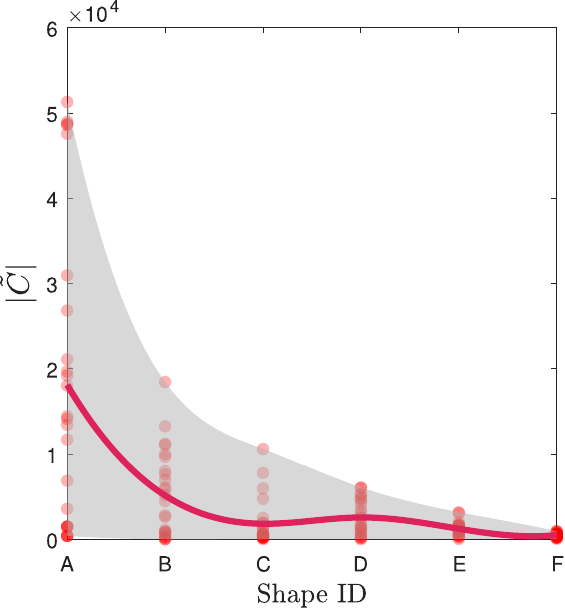}
    \caption*{(\rom{2})}
  \end{subfigure}}
  \caption{\rom{1}) The six shapes \textit{A} through \textit{F} using spherical harmonics, and \rom{2}) $|\tilde{C}|$ plotted against shapes \textit{A} through \textit{F}, the red curve passes through mean values, the shaded area gives the range}
  \label{sph_harm}
\end{figure}

\begin{figure}[b]
\centering
\begin{subfigure}[b]{\linewidth}
    \centering
    \includegraphics[width=0.95\linewidth,keepaspectratio]{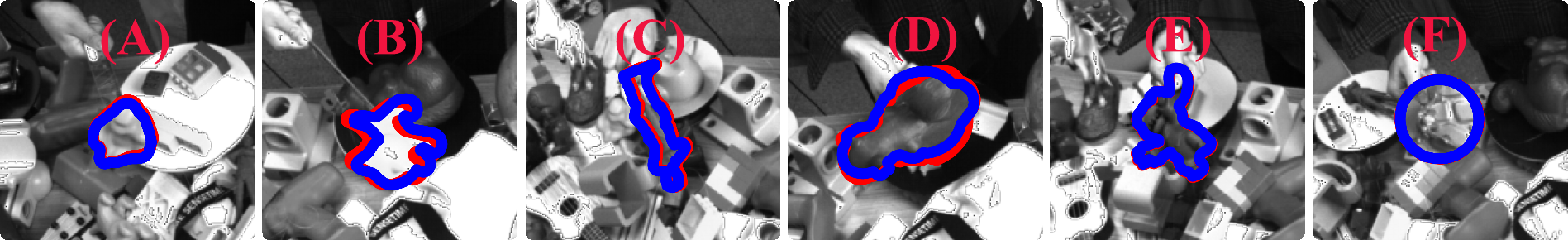}
  \end{subfigure}

  \caption{Qualitative results on the asymmetric shapes from \textit{complex\_movable\_handheld} sequence of \y{bcot} dataset: \textit{Ape}, \textit{Cat}, \textit{Jack}, \textit{Squirrel}, \textit{Stitch}, and \textit{Vampire queen} shown in (A) through (F); red curves are \y{gt} silhouettes, blue curves are from \y{gppersp} (please zoom-in for details)}
  \label{bcot_qual}
\end{figure}

\section{Discussion}

Despite our method's accuracy, strong occlusion or heavy noise can still induce failures. Such settings are fundamentally unsolvable, no existing silhouette-only approach is reliable under such data corruption. Our contribution targets the regime where silhouettes alone remain informative, and within this scope, our method delivers consistently strong performance. Section~8 (supplementary) collects \textit{additional clarifications and responses to typical questions}.

\section{Conclusion}\label{sec_concl}



We introduce the first globally optimal solution to the \y{pfs} problem, delivering high accuracy and broad applicability across domains such as robotics, medical imaging, and \y{ar}. While extreme occlusion and high noise remain unsolvable theoretical bottlenecks, within the practically relevant regime of unoccluded silhouettes, our approach decisively advances the field, establishing benchmarks that surpass all prior methods.
{
    \clearpage
    \section*{Acknowledgements}
    \addcontentsline{toc}{section}{Acknowledgements}
    This research was conducted in the research campus MODAL, funded by the Federal Ministry of Research, Technology and Space (BMFTR), Germany, grant no. 3FO18501. This research is also connected with the Competence Center for Excellent Technologies (COMET , grant no. 911654), administered by the Austrian Research Promotion Agency (FFG).
    \small
    \bibliography{main}

@article{corsini2012efficient,
  title={Efficient and flexible sampling with blue noise properties of triangular meshes},
  author={Corsini, Massimiliano and Cignoni, Paolo and Scopigno, Roberto},
  journal={IEEE transactions on visualization and computer graphics},
  volume={18},
  number={6},
  pages={914--924},
  year={2012},
  publisher={IEEE}
}

@article{edelsbrunner1983shape,
  title={On the shape of a set of points in the plane},
  author={Edelsbrunner, Herbert and Kirkpatrick, David and Seidel, Raimund},
  journal={IEEE Transactions on information theory},
  volume={29},
  number={4},
  pages={551--559},
  year={1983},
  publisher={IEEE}
}

@article{hartley2009global,
  title={Global optimization through rotation space search},
  author={Hartley, Richard I and Kahl, Fredrik},
  journal={International Journal of Computer Vision},
  volume={82},
  number={1},
  pages={64--79},
  year={2009},
  publisher={Springer}
}

@misc{Wikipedia_2024, 
title={Azimuthal equidistant projection},
author={Wikipedia},
url={https://en.wikipedia.org/wiki/Azimuthal_equidistant_projection}, journal={Wikipedia}, publisher={Wikimedia Foundation}, 
note = "[Online; accessed 10-April-2026]", year={2024}, month={Dec}}

@book{matlab_func, 
place={Natick, MA}, 
title={MATLAB Function Reference}, publisher={The Mathworks, Inc.}, author={The Mathworks, Inc.}, year={2024}}

@article{bektas2015least,
  title={Least squares fitting of ellipsoid using orthogonal distances},
  author={Bektas, Sebahattin},
  journal={Boletim de ci{\^e}ncias geod{\'e}sicas},
  volume={21},
  number={2},
  pages={329--339},
  year={2015},
  publisher={SciELO Brasil}
}

@Article{manopt,
    author  = {Boumal, N. and Mishra, B. and Absil, P.-A. and Sepulchre, R.},
    journal = {Journal of Machine Learning Research},
    title   = {{M}anopt, a {M}atlab Toolbox for Optimization on Manifolds},
    year    = {2014},
    number  = {42},
    pages   = {1455--1459},
    volume  = {15},
    url     = {https://www.manopt.org}
}

@inproceedings{zhang2009using,
  title={Using silhouette for pose estimation of object with surface of revolution},
  author={Zhang, Ming and Zheng, Yinqiang and Liu, Yuncai},
  booktitle={2009 16th IEEE International Conference on Image Processing (ICIP)},
  pages={333--336},
  year={2009},
  organization={IEEE}
}

@inproceedings{toppe2011silhouette,
  title={Silhouette-based variational methods for single view reconstruction},
  author={T{\"o}ppe, Eno and Oswald, Martin R and Cremers, Daniel and Rother, Carsten},
  booktitle={Video Processing and Computational Video: International Seminar, Dagstuhl Castle, Germany, October 10-15, 2010. Revised Papers},
  pages={104--123},
  year={2011},
  organization={Springer}
}

@inproceedings{gummeson2024relative,
  title={Relative pose from cylinder silhouettes},
  author={Gummeson, Anna and Oskarsson, Magnus},
  booktitle={Proceedings of the Asian Conference on Computer Vision},
  pages={2545--2561},
  year={2024}
}

@inproceedings{prasad2005fast,
  title={Fast and Controllable 3D Modelling From Silhouettes.},
  author={Prasad, Mukta and Fitzgibbon, Andrew W and Zisserman, Andrew},
  booktitle={Eurographics (Short Presentations)},
  pages={9--12},
  year={2005}
}

@inproceedings{kong2017using,
  title={Using locally corresponding cad models for dense 3d reconstructions from a single image},
  author={Kong, Chen and Lin, Chen-Hsuan and Lucey, Simon},
  booktitle={Proceedings of the IEEE conference on computer vision and pattern recognition},
  pages={4857--4865},
  year={2017}
}

@inproceedings{menudetmodel,
  title={Model-based Shape from Silhouette: A Solution Involving a Small Number of Views},
  author={Menudet, JF and Becker, JM and Fournel, T and Mennessier, C},
  booktitle={Proceedings of the Second International Conference on Computer Vision Theory and Applications},
  pages={379--386},
  year={2007}
}

@article{gaudilliere2023perspective,
  title={Perspective-1-Ellipsoid: Formulation, analysis and solutions of the camera pose estimation problem from one ellipse-ellipsoid correspondence},
  author={Gaudilli{\`e}re, Vincent and Simon, Gilles and Berger, Marie-Odile},
  journal={International Journal of Computer Vision},
  volume={131},
  number={9},
  pages={2446--2470},
  year={2023},
  publisher={Springer}
}

@article{howe2007silhouette,
  title={Silhouette lookup for monocular 3d pose tracking},
  author={Howe, Nicholas R},
  journal={Image and Vision Computing},
  volume={25},
  number={3},
  pages={331--341},
  year={2007},
  publisher={Elsevier}
}

@inproceedings{howe2004silhouette,
  title={Silhouette lookup for automatic pose tracking},
  author={Howe, Nicholas R},
  booktitle={2004 Conference on Computer Vision and Pattern Recognition Workshop},
  pages={15--22},
  year={2004},
  organization={IEEE}
}

@inproceedings{wang2023deep,
  title={Deep active contours for real-time 6-DoF object tracking},
  author={Wang, Long and Yan, Shen and Zhen, Jianan and Liu, Yu and Zhang, Maojun and Zhang, Guofeng and Zhou, Xiaowei},
  booktitle={Proceedings of the IEEE/CVF International Conference on Computer Vision},
  pages={14034--14044},
  year={2023}
}

@article{chen2025robust,
  title={Robust 6DoF Pose Tracking Considering Contour and Interior Correspondence Uncertainty for AR Assembly Guidance},
  author={Chen, Jixiang and Chen, Jing and Liu, Kai and Chang, Haochen and Fu, Shanfeng and Yang, Jian},
  journal={arXiv preprint arXiv:2502.11971},
  year={2025}
}

@inproceedings{cui2024silhouette,
  title={Silhouette-Based 6D Object Pose Estimation},
  author={Cui, Xiao and Li, Nan and Zhang, Chi and Zhang, Qian and Feng, Wei and Wan, Liang},
  booktitle={International Conference on Computational Visual Media},
  pages={157--179},
  year={2024},
  organization={Springer}
}

@inproceedings{albanis2020dronepose,
  title={Dronepose: photorealistic uav-assistant dataset synthesis for 3d pose estimation via a smooth silhouette loss},
  author={Albanis, Georgios and Zioulis, Nikolaos and Dimou, Anastasios and Zarpalas, Dimitrios and Daras, Petros},
  booktitle={Computer Vision--ECCV 2020 Workshops: Glasgow, UK, August 23--28, 2020, Proceedings, Part II 16},
  pages={663--681},
  year={2020},
  organization={Springer}
}

@inproceedings{wang2020directshape,
  title={Directshape: Direct photometric alignment of shape priors for visual vehicle pose and shape estimation},
  author={Wang, Rui and Yang, Nan and Stueckler, Joerg and Cremers, Daniel},
  booktitle={2020 IEEE International Conference on Robotics and Automation (ICRA)},
  pages={11067--11073},
  year={2020},
  organization={IEEE}
}

@inproceedings{hebert2012combined,
  title={Combined shape, appearance and silhouette for simultaneous manipulator and object tracking},
  author={Hebert, Paul and Hudson, Nicolas and Ma, Jeremy and Howard, Thomas and Fuchs, Thomas and Bajracharya, Max and Burdick, Joel},
  booktitle={2012 IEEE International Conference on Robotics and Automation},
  pages={2405--2412},
  year={2012},
  organization={IEEE}
}

@article{collins2020augmented,
  title={Augmented reality guided laparoscopic surgery of the uterus},
  author={Collins, Toby and Pizarro, Daniel and Gasparini, Simone and Bourdel, Nicolas and Chauvet, Pauline and Canis, Michel and Calvet, Lilian and Bartoli, Adrien},
  journal={IEEE Transactions on Medical Imaging},
  volume={40},
  number={1},
  pages={371--380},
  year={2020},
  publisher={IEEE}
}

@article{guo2021pose,
  title={Pose initialization of uncooperative spacecraft by template matching with sparse point cloud},
  author={Guo, Wulong and Hu, Weiduo and Liu, Chang and Lu, Tingting},
  journal={Journal of Guidance, Control, and Dynamics},
  volume={44},
  number={9},
  pages={1707--1720},
  year={2021},
  publisher={American Institute of Aeronautics and Astronautics}
}

@article{zhang2023mc,
  title={MC-LRF based pose measurement system for shipborne aircraft automatic landing},
  author={Zhang, Zhuo and Qiufu, WANG and Daoming, BI and Xiaoliang, SUN and Qifeng, YU},
  journal={Chinese Journal of Aeronautics},
  volume={36},
  number={8},
  pages={298--312},
  year={2023},
  publisher={Elsevier}
}

@article{perez2023optimal,
  title={Optimal coherent point selection for 3D quality inspection from silhouette-based reconstructions},
  author={P{\'e}rez Soler, Javier and Guardiola, Jose-Luis and Perez Jimenez, Alberto and Garrigues Carb{\'o}, Pau and Garc{\'\i}a Sastre, Nicol{\'a}s and Perez-Cortes, Juan-Carlos},
  journal={Mathematics},
  volume={11},
  number={21},
  pages={4419},
  year={2023},
  publisher={MDPI}
}

@misc{phlegm, 
title={3D model of a dragon released during EuroGraphics 2007}, 
url={www.dcgi.fel.cvut.cz/eg07/index.php?page=dragon}, 
journal={Sketchfab, Inc.}, 
author={Jiří Filip and Holub, Radek and Havran, Vlastimil and Křivánek, Jaroslav and Sýkora, Daniel},
year={2007}}

@inproceedings{more2006levenberg,
  title={The Levenberg-Marquardt algorithm: implementation and theory},
  author={Mor{\'e}, Jorge J},
  booktitle={Numerical analysis: proceedings of the biennial Conference held at Dundee, June 28--July 1, 1977},
  pages={105--116},
  year={2006},
  organization={Springer}
}

@book{gilmore2006lie,
  title={Lie groups, Lie algebras, and some of their applications},
  author={Gilmore, Robert},
  year={2006},
  publisher={Courier Corporation}
}

@article{ugray2007scatter,
  title={Scatter search and local NLP solvers: A multistart framework for global optimization},
  author={Ugray, Zsolt and Lasdon, Leon and Plummer, John and Glover, Fred and Kelly, James and Mart{\'\i}, Rafael},
  journal={INFORMS Journal on computing},
  volume={19},
  number={3},
  pages={328--340},
  year={2007},
  publisher={Informs}
}

@inproceedings{li2022bcot,
  title={Bcot: A markerless high-precision 3d object tracking benchmark},
  author={Li, Jiachen and Wang, Bin and Zhu, Shiqiang and Cao, Xin and Zhong, Fan and Chen, Wenxuan and Li, Te and Gu, Jason and Qin, Xueying},
  booktitle={Proceedings of the IEEE/CVF Conference on Computer Vision and Pattern Recognition},
  pages={6697--6706},
  year={2022}
}

@inproceedings{Bochkovskii2024:arxiv,
  author     = {Aleksei Bochkovskii and Ama\"{e}l Delaunoy and Hugo Germain and Marcel Santos and
               Yichao Zhou and Stephan R. Richter and Vladlen Koltun},
  title      = {Depth Pro: Sharp Monocular Metric Depth in Less Than a Second},
  booktitle  = {International Conference on Learning Representations},
  year       = {2025},
  url        = {https://arxiv.org/abs/2410.02073},
}

@misc{wiki:Spherical_harmonics,
   author = "Wikipedia",
   title = "{Spherical harmonics}",
   year = "2025",
   howpublished = {\url{http://en.wikipedia.org/w/index.php?title=Spherical\%20harmonics&oldid=1321561306}},
   note = "[Online; accessed 13-November-2025]"
 }

@article{vijayakumar1998invariant,
  title={Invariant-based recognition of complex curved 3D objects from image contours},
  author={Vijayakumar, B and Kriegman, David and Ponce, Jean},
  journal={Computer Vision and Image Understanding},
  volume={72},
  number={3},
  pages={287--303},
  year={1998},
  publisher={Elsevier}
}

@inproceedings{lazebnik2002pencils,
  title={On pencils of tangent planes and the recognition of smooth 3d shapes from silhouettes},
  author={Lazebnik, Svetlana and Sethi, Amit and Schmid, Cordelia and Kriegman, David and Ponce, Jean and Hebert, Martial},
  booktitle={European Conference on Computer Vision},
  pages={651--665},
  year={2002},
  organization={Springer}
}

@article{horn1988closed,
  title={Closed-form solution of absolute orientation using orthonormal matrices},
  author={Horn, Berthold KP and Hilden, Hugh M and Negahdaripour, Shahriar},
  journal={Journal of the Optical Society of America A},
  volume={5},
  number={7},
  pages={1127--1135},
  year={1988},
  publisher={Optical Society of America}
}

@article{bartoli2025camera,
  title={Camera pose in SfT and NRSfM under isometric and weaker deformation models},
  author={Bartoli, Adrien and Sengupta, Agniva},
  journal={Computer Vision and Image Understanding},
  pages={104488},
  year={2025},
  publisher={Elsevier}
}

@article{sengupta2025convex,
  title={Convex Solutions to SfT and NRSfM under Algebraic Deformation Models},
  author={Sengupta, Agniva and Bartoli, Adrien},
  journal={IEEE Transactions on Pattern Analysis and Machine Intelligence},
  year={2025},
  publisher={IEEE}
}

@article{sengupta2025shape,
  title={Shape-from-template with generalised camera},
  author={Sengupta, Agniva and Zachow, Stefan},
  journal={Image and Vision Computing},
  pages={105579},
  year={2025},
  publisher={Elsevier}
}
}

\end{document}